\documentclass[11pt, letterpaper, logo, onecolumn, copyright]{template_from_google}
\usepackage[T1]{fontenc}

\usepackage{bm} % For bold math symbols like \bm
\usepackage{nicefrac} % For typesetting compact fractions

\usepackage{booktabs} % For professional quality tables
\usepackage{enumitem} % For list customization
\usepackage{fancyhdr}
\usepackage{multirow} % For multi-row cells in tables
\usepackage{tabularx}
\usepackage{multicol}
\usepackage{array}
\usepackage{longtable}
\usepackage{lipsum} % For placeholder text
\usepackage{wrapfig} % To wrap text around figures
\usepackage{tocloft} % To control the Table of Contents
\usepackage{fancybox}
\usepackage{authblk} % For author affiliations

\usepackage{algorithm}
\usepackage[noend]{algpseudocode}
\usepackage{graphicx}
\usepackage{caption}
\usepackage[dvipsnames]{xcolor}
\usepackage{xspace} % For intelligent spacing after macros

\usepackage[authoryear, sort&compress, round]{natbib}

\usepackage[dvipsnames]{xcolor} % Make sure xcolor is loaded

\usepackage{hyperref}

\definecolor{darkblue}{rgb}{0.0, 0.0, 0.6}
\definecolor{darkred}{rgb}{0.7, 0.0, 0.0}
\hypersetup{
  pdffitwindow=true,
  pdfstartview={FitH},
  pdfnewwindow=true,
  colorlinks,
  linktocpage=true,
  linkcolor=darkred,
  urlcolor=darkblue,
  citecolor=darkblue
}

\usepackage{graphicx}
\usepackage{amsmath,amsfonts,amsthm}
\usepackage{bbm}
\usepackage{algorithm}
\usepackage[noend]{algpseudocode}

\usepackage{multirow}
\usepackage{xspace}
\usepackage{rotating}
\usepackage{url}

\newcommand{\method}{SQS\xspace}

\newtheorem{condition}{Condition}
\newtheorem{theorem}{Theorem}

\newtheorem{lemma}{Lemma}

\newtheorem{definition}{Definition}

\newcommand{\E}{\mathbb{E}}

\newcommand{\R}{\mathbb{R}}
\newcommand{\KL}{\mathrm{KL}}

\DeclareMathOperator*{\argmax}{arg\,max}
\DeclareMathOperator*{\argmin}{arg\,min}

\newcommand{\norm}[1]{\left\|#1\right\|}
\newcommand{\Exp}[1]{\exp\left(#1\right)}

\makeatletter

\renewcommand\paragraph{\@startsection{paragraph}{4}{\z@}%
            {-2.5ex\@plus -1ex \@minus -.25ex}%
            {1.25ex \@plus .25ex}%
            {\itshape\normalsize\bfseries}}
\makeatother

\let\cite\citep

\title{SQS: Bayesian DNN Compression through Sparse Quantized Sub-distributions}
\newcommand{\shorttitle}{SQS}
\reportnumber{} % Leave blank if n/a

\renewcommand{\today}{}

\author[1,$\ast$]{Ziyi Wang}
\author[2,$\ast$]{Nan Jiang}
\author[1]{Guang Lin}
\author[1]{Qifan Song}

\affil[1]{Purdue University, IN, USA}
\affil[2]{University of Texas at El Paso, TX, USA}

\renewcommand{\copyrightext}{$^\ast$Both authors contributed equally to this work.\\ Emails: \{wang4538,guanglin,qfsong\}@purdue.edu, njiang@utep.edu}

\begin{abstract}

Compressing large-scale neural networks is essential for deploying models on resource-constrained devices. Most existing methods adopt weight pruning or low-bit quantization individually, often resulting in suboptimal compression rates to preserve acceptable performance drops.
We introduce a unified framework for simultaneous pruning and low-bit quantization via Bayesian variational learning (\method), which achieves higher compression rates than prior baselines while maintaining comparable performance.
The key idea is to employ a spike-and-slab prior to induce sparsity and model quantized weights using Gaussian Mixture Models (GMMs) to enable low-bit precision.
Due to the intractability of the objective involving spike-and-slab priors with GMMs, we derive an efficient approximation that facilitates effective compression with minimal accuracy loss.
In theory, we provide a consistent result for our proposed variational approach to a sparse and quantized deep neural network. 
Extensive experiments on compressing ResNet, BERT-base, Llama3.2, and Qwen2.5 models show that our method achieves higher compression rates than a line of existing methods with comparable performance drops. \\ Project page: \url{https://comeusr.github.io/SQS_Webpage/}.

\end{abstract}

\begin{document}

\maketitle

\section{Introduction}
Deep Neural Networks (DNNs) have achieved state-of-the-art performance across a wide range of tasks but at the cost of significantly increased computational and memory requirements~\citep{radford2018improving, xu2020bert, touvron2023llama,kumar2025scaling}, making deployment on resource-constrained devices challenging. Model compression methods have therefore been proposed to reduce the size and computational complexity of DNNs while maintaining predictive accuracy, including pruning~\citep{lecun1989optimal, han2015deep}, weight quantization~\citep{courbariaux2015binaryconnect, rastegari2016xnor, frantar2022gptq, lin2024awqactivationawareweightquantization}, knowledge distillation~\citep{Park_2019_CVPR, gou2021knowledge}, and neural architecture search~\citep{liu2018progressive, wang2020apq}.

Among these, \emph{weight pruning} and \emph{low-bit quantization} are particularly effective and widely adopted for compressing DNNs~\citep{bucilua2006model, choudhary2020comprehensive,liu2025lowbitmodelquantizationdeep}. 
Weight pruning eliminates redundant or unimportant weights by setting selected weights to zero, thereby reducing the number of active parameters without significantly altering the model architecture~\citep{you2019gate, guo2016dynamic, dong2017learning}. 
On the other hand, quantization reduces the bit-width of numerical representations for inputs, outputs, and weights by converting high-precision formats (e.g., \texttt{FP32}) to lower-precision alternatives, such as \texttt{FP8} or \texttt{INT8}. 
This quantization coarsens the model representation and yields significant reductions in memory footprint and computational overhead. It enhances efficiency in both training and inference across diverse architectures, including ResNet~\citep{banner2018scalable}, Transformers~\citep{sun2019hybrid}, Large language models~\citep{DBLP:conf/nips/DettmersPHZ23,DBLP:conf/icml/WangG0ZYGZ025}, and vision-language models~\citep{wortsman2023stable}.

However, quantization and pruning inevitably introduce distributional shifts from the original DNNs, often leading to performance degradation~\citep{dong2022finding}. To mitigate this, existing methods adopt conservative compression rates, limiting their applicability to resource-constrained environments~\citep{wang2020differentiable, wang2020apq, bai2022dual, frantar2022optimal, bai2023unified}. Achieving high compression rates while maintaining acceptable performance remains an open question to explore.

To tackle the above problem, we introduce a unified framework: \textbf{S}parse \textbf{Q}uantized \textbf{S}ub-distribution compression (\method),  which unifies pruning and quantization within a single variational learning process. 
Instead of applying pruning and quantization separately, the key idea of \method is joint pruning and quantization that learns a sparse, quantized sub-distribution over network weights through variational learning.
To model the variational posterior, we adopt a spike-and-slab prior combined with a Gaussian Mixture Model (GMM): the spike component encourages sparsity for pruning, while the GMM component models a quantized weight distribution, effectively mitigating performance degradation. The training pipeline of \method is in Figure~\ref{fig:method}. Theoretically, we show that under mild conditions, our \method method finds a sparse and quantized neural network that converges to the true underlying target neural network with high probability.

In our experiments, we compare several recent state-of-the-art compression methods across a range of widely used neural networks, including ResNet, BERT-base, Llama3.2, and Qwen2.5. Our findings show that (1) under the same bit-width setting, our \method achieves the highest compression rate, requiring fewer parameters than baselines. (2) At the same compression rate, our \method achieves the smallest accuracy drop or F1 score drop among all approaches, with particularly strong performance at 2-bit and 4-bit precision.  

% Furthermore, our ablation study demonstrates that 1) the spike-and-slab distribution is more effective in encouraging sparsity than the Gaussian alternative, yielding better scores on sparsity metrics.  2) Our \method with Bayesian average yields a better inference result than using the greedy approach. 3) Our method effectively preserves weight distribution outliers with an outlier-aware window strategy, yielding better performance than the even-window approach. 

Further ablation studies highlight the contributions of individual components:
(1) The spike-and-slab distribution is more effective in promoting sparsity than Gaussian alternatives.
(2) Bayesian averaging during inference outperforms greedy weight selection.
(3) An outlier-aware window strategy better preserves informative weight outliers compared to uniform windowing, further improving performance.%\footnote{Implementation of \method is available at: \url{https://nan-jiang-group.github.io/SQS/}}

Our contributions are summarized as follows:
\begin{itemize}[
    leftmargin=*,
    topsep=2pt,
    itemsep=1pt,
    % parsep=0pt
]
\item  We propose \method, a unified Bayesian framework for compressing full-precision DNNs into sparse, low-bit models. Unlike methods that perform pruning and quantization separately, \method jointly learns which weights to remove and how to quantize the remaining weights using a spike-and-GMM variational distribution.

\item We derive a tractable approximate objective for training \method and provide theoretical guarantees showing that, under mild conditions, the learned sparse and quantized network converges to the target regression function.

\item We conduct experiments on ResNet, BERT-base, Llama3.2, and Qwen2.5. Across these architectures, \method achieves higher compression rates with comparable or smaller accuracy degradation than existing baselines. Ablation studies further validate the effectiveness of the spike-and-slab formulation, Bayesian averaging at inference time, and the outlier-aware windowing strategy.
\end{itemize}

\section{Preliminaries}

\textbf{Low-bit Quantization} uses discrete low-bit values to approximate full-precision floating-point values, primarily to reduce precision for more efficient storage and computation while preserving essential information~\citep{gholami2022survey}.  Formally, it is defined as a mapping $Q:  \R \rightarrow \mathcal{Q} = \{\mu_1, \dots, \mu_K\}$, where the input is the full-precision weight and $\mathcal{Q} $ denotes the set of low-bit discrete values.  Representative quantization methods include deterministic quantization~\citep{jacob2018quantization}, stochastic quantization~\citep{courbariaux2015binaryconnect}, and end-to-end learnable quantization~\citep{dong2022finding}.

Specifically, let $\theta=(\theta_1,\ldots,\theta_T)\in\mathbb{R}^T$ represent the pre-trained full-precision weights of a deep neural network, with $\theta_i$ denoting the $i$-th weight.
Given a quantization set $\mathcal{Q}=\{\mu_1, \dots, \mu_K\}$,  a general stochastic quantization is a map $Q:\mathbb{R}\rightarrow \mathcal{P}_K(\mathbb{R})$ from the real numbers to the space of probability distributions over $\mathbb{R}$ with finite support of cardinality $K$.
For each weight $\theta_i$, 
\begin{align*}
    Q(\theta_i) = \mu_k, &\qquad\text{with probability } p_{k_i}, 
\end{align*}
for $ i=1,\dots, T$. Here $\mu_k$ is the learnable parameter and $p_{k_i}$ is the corresponding probability that weight $\theta_i$ is quantized to weight $\mu_k$. A key challenge is the distribution divergence between the quantized weights and the original weights, leading to significant performance degradation \citep{dong2022finding}. To mitigate this, \citet{dong2022finding} propose to approximate the quantized weight distribution using a Gaussian Mixture Model (GMM):
\begin{equation}\label{eq:GMM}
Q(\theta_i) \approx \sum_{k=1}^{K} \phi_k(\theta_i; \pi_k) \mathcal{N}(\mu_k, \sigma^2_k),
\end{equation}
where \( \mathcal{N}(\mu_k, \sigma_k^2) \) denotes a Gaussian distribution, and \( \phi_k(\theta_i; \pi_k)
\) is the weight for the $k$-th Gaussian component $\mathcal{N}(\mu_k, \sigma^2_k)$. 
To control the sharpness of this mixture, a temperature-scaled softmax is applied to obtain \( \phi_k(\theta_i; \pi_k)\), that is: %for $k=1,\ldots, K$,
\begin{align} \label{eq:compute-phi-k}
    \phi_k(\theta_i; \pi_k) =\frac{\Exp{\varphi_k(\theta_i; \pi_k)/\tau}}{\sum_{j=1}^K\Exp{\varphi_j(\theta_i;\pi_j)/\tau}},
\end{align}
where the temperature parameter \( \tau > 0 \) controls the concentration of the distribution. As \( \tau \rightarrow 0 \), the GMM in Equation~\eqref{eq:GMM} approaches a single dominant Gaussian component.
Given a prior distribution $(\pi_1, \dots, \pi_K)$ over the quantization set $\mathcal{Q}$, the posterior component weight 
$\varphi_k(\theta_i; \pi_k)$ is: 
\begin{align*}
 \varphi_k(\theta_i; \pi_k)  = \frac{\exp({\pi_k \mathcal{N}(\theta_i\mid \mu_k, \sigma^2_k)})}{\sum_{j=1}^{K}\exp({\pi_j\mathcal{N}(\theta_i \mid \mu_j, \sigma^2_j)})}.
\end{align*}
Additionally, with sufficiently small \( \sigma_k^2 \), the GMM approximates a multinomial distribution over \( \mathcal{Q} \), effectively bridging continuous and discrete quantization.  For simplicity, we denote $\phi_k(\theta_i)$ as a shorthand for $\phi_k(\theta_i; \pi_k)$ throughout the remainder of this paper. 

In our experiments, we find that the GMM-based compression method~\citep{dong2022finding} still cannot achieve a high compression rate while maintaining a small performance drop, as it cannot efficiently encourage sparsity during training.

\textbf{Variational learning.} Given an observed dataset $D$, the goal of a Bayesian framework is to infer the true posterior distribution $\pi( \mathbf{\theta} | D) \propto \pi(\mathbf{\theta}) p(D|\mathbf{\theta})$, where $\pi(\mathbf{\theta})$ denotes the prior and $p(D|\mathbf{\theta})$ the likelihood.
Since the posterior is generally intractable, variational learning~\citep{jordan1999introduction} is proposed to approximate it by selecting the closest distribution from a variational family $\mathcal{F}$ in terms of the Kullback–Leibler (KL) divergence~\citep{csiszar1975divergence}:
\begin{equation}\label{eq:variational}
    q^{*}(\theta) \;=\; \argmin_{q(\theta) \in \mathcal{F}} \,\KL\!\left(q(\theta) \,\|\, \pi(\theta \mid D)\right).
\end{equation}
Following~\cite{blei2017variational}, this optimization is equivalent to minimizing the negative Evidence Lower Bound (ELBO), defined as:
\begin{equation}\label{eq:elbo}
    \Omega(q)  :=  -\E_{q(\theta)}[\log p(D|{\theta})]+\KL(q(\theta) \| \pi(\theta)),
\end{equation}
where the first term measures how well the variational distribution $q(\theta)$ aligns with the log-likelihood of the observed data, and the second term regularizes $q(\theta)$ to stay close to the prior $ \pi(\theta)$.

Our \method method employs a variational family based on a spike-and-GMM distribution to approximate the sparse and quantized posterior. The first term in Equation~\eqref{eq:elbo} allows the spike-and-GMM to learn the posterior distribution given the data. For the second term, we adopt a spike-and-slab prior $\pi(\cdot)$ distribution to promote sparsity in the network weights. %

\section{Methodology}\label{sec:method} 
The objective is to approximate a full-precision neural network $f(\cdot; \theta)$ with a Bayesian model $f(\cdot; \tilde{\theta})$ that is both sparse and low-precision,
while minimizing performance degradation. 
To achieve this, we employ a spike-and-slab distribution combined with a GMM to parameterize the variational posterior.

\begin{figure*}[!t]
    \centering
    \includegraphics[width=\linewidth]{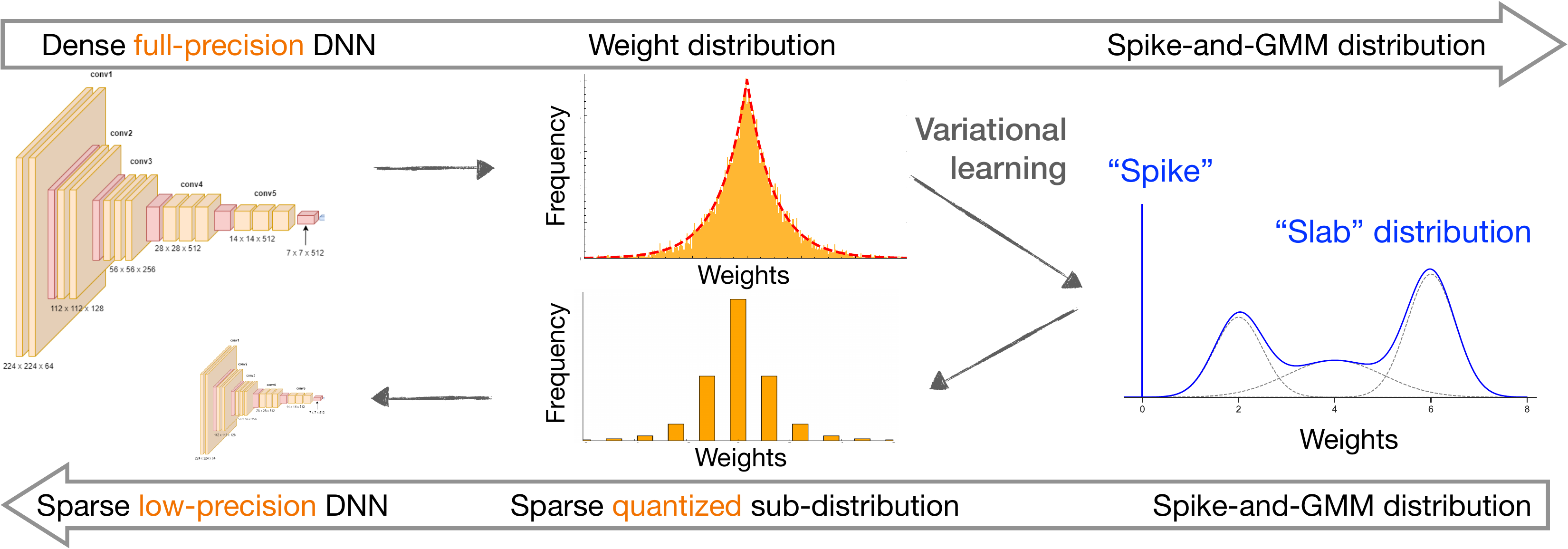}

     \caption{Our \method method achieves high compression with minimal performance degradation by jointly pruning and quantizing model weights through variational learning. We employ a spike-and-GMM variational distribution to approximate full-precision weights: the \textit{spike} component promotes sparsity for pruning, while the \textit{slab} component (i.e., GMM) models a quantized weight distribution.}
	\label{fig:method}
\end{figure*}

\subsection{SQS: Variational learning for sparse and quantized sub-distribution}

% ========
The spike-and-slab prior consists of a point mass at zero (spike) and a continuous distribution (slab)~\citep{bai2020efficient, ishwaran2005spike}. Formally, let $\gamma = (\gamma_1, \ldots, \gamma_T)$ be a binary indicator vector, where each $\gamma_i$ determines whether the corresponding weight $\theta_i$ is preserved ($\gamma_i=1$) or pruned ($\gamma_i=0$). The prior for each weight $\tilde{\theta}_i$ is defined as:
\begin{align*}
\gamma_i \;&\sim \; \mathrm{Bern}(\lambda), \\
\tilde{\theta}_i \mid \gamma_i \; &\sim \;\gamma_i \mathcal{N}(0, \sigma_0^2) + (1 - \gamma_i) \delta_0,
\end{align*}
where $\lambda$ is the prior probability of retaining a weight, and $\sigma_0^2$ is the prior variance of the Gaussian slab. Marginalizing out the binary variable $\gamma_i$, the prior distribution over $\tilde{\theta}_i$ becomes:
\begin{equation}\label{eq:marginal_prior}
\pi(\tilde{\theta}_i) = \lambda \mathcal{N}(0, \sigma_0^2) + (1 - \lambda) \delta_0,
\end{equation}
where $1 - \lambda$ corresponds to the prior pruning probability. For example, in a DNN with a target sparsity of $90\%$, setting $\lambda = 0.1$ implies that each weight has a $90\%$ prior probability of being pruned.

\subsubsection{Training Procedure}
To incorporate quantization into the variational family, we extend the spike-and-slab formulation by modeling the slab using a $K$-component GMM. Each variational distribution $q(\theta_i)$ is then defined as:
%
% Building upon the GMM quantization in \eqref{eq:GMM}, one natural idea is to combine this distribution with Dirac distribution $\delta_0$ to form a variational family $\mathcal{F}$. 
% That is, any $q(\theta) \in \mathcal{F}$ has the following form: 
\begin{align*}
\gamma_i &\sim \mathtt{Bern}(\tilde{\lambda}_i)\\
    \tilde{\theta}_i | \gamma_i &\sim \gamma_i   \sum_{k=1}^{K} \phi_k(\theta_i) \mathcal{N}(\mu_k, \sigma^2_k)
    +(1-\gamma_i)\delta_0
\end{align*}
where $\phi_k(\theta_i)$ is the mixture weight for component $k$, and $\tilde{\lambda}_i$ is the variational probability of retaining weight $\theta_i$. The marginal variational distribution $q(\tilde{\theta}_i)$ is:
\begin{equation}\label{eq:marginal_vb}
    q(\tilde{\theta}_i)=\tilde{\lambda}_i \sum_{k=1}^{K} \phi_k(\theta_i) \mathcal{N}(\mu_k, \sigma^2_k) + (1-\tilde{\lambda}_i) \delta_0.
\end{equation}
Given this variational family, we define the learning objective based on the ELBO:
\begin{equation}\label{eq:exact_ELBO}
     \Omega(\tilde{\theta})= -\E_{q(\tilde{\theta})}\left[\log p(D|\tilde{\theta})\right]+\sum_{i=1}^{T}\KL\left(q(\tilde{\theta}_i) \| \pi(\tilde{\theta}_i)\right).
\end{equation}
Yet, computing Equation~\eqref{eq:exact_ELBO} is intractable, as no closed-form solution exists for the KL divergence between $q(\tilde{\theta}_i)$ and the spike-and-slab prior $\pi(\tilde{\theta}_i)$.
To overcome this challenge, we propose the following approximation:
For each coordinate $i$, the posterior mean is
$\E_{q(\tilde{\theta}_i)}[\tilde{\theta}_i]
=\tilde{\lambda}_i\sum_{k=1}^{K}\mu_k\phi_k(\theta_i;\pi,\tau)$.
Collecting these coordinate-wise means, define
\begin{equation*}
    \theta^{\mathrm{avg}} 
    :=\E_{q(\tilde{\theta})}[\tilde{\theta}]
    =\left(
    \tilde{\lambda}_i\sum_{k=1}^{K}\mu_k\phi_k(\theta_i;\pi,\tau)
    \right)_{i=1}^{T}
    \in\mathbb{R}^{T}.
\end{equation*}
The approximate objective becomes
\begin{align}
    \Omega_{\mathtt{apx}}(\tilde{\theta})
    =&-\log p(D\mid\theta^{\mathrm{avg}} )
    +\sum_{i=1}^{T}
    \KL\left(
    \mathtt{Bern}(\tilde{\lambda}_i)
    \,\|\,\mathtt{Bern}(\lambda)
    \right)
    +\sum_{i=1}^{T}\tilde{\lambda}_i
    \KL\left(
    \mathcal{N}(\mu_{k_i^*},\sigma^2_{k_i^*})
    \,\|\,\mathcal{N}(0,\sigma_0^2)
    \right),
    \label{eq:approx_elbo}
\end{align}
where $k_i^*=\argmax_{1\leq k\leq K}\phi_k(\theta_i)$.
The first term uses the plug-in approximation $-\E_{q(\tilde{\theta})}
[\log p(D\mid\tilde{\theta})]
\approx-\log p(D\mid\theta^{\mathrm{avg}} ).$ 
Thus, the likelihood is evaluated using the complete parameter vector $\theta^{\mathrm{avg}} $.
The second and the third terms provide an upper bound on the term $\sum_{i=1}^{T}\KL\left(q(\tilde{\theta}_i) \| \pi(\tilde{\theta}_i)\right)$ by applying Lemma \ref{lem:mix}.
Please refer to Appendix~\ref{apx:obj-approx-theory} for a detailed derivation of Equation~\eqref{eq:approx_elbo}.

\subsubsection{Inference Procedure} 
In the inference stage, we first sample the sparse and quantized weights given the learned parameters and predict the output $\widehat{y}$ for each testing input $x$. Let $\widehat{q}(\cdot) $ denote the optimization solution of the above variational learning, associated with the optimal parameter estimations $\{\hat\mu_i, \hat\sigma^2_i, \hat\lambda_i\}_{i=1}^T$,  and the corresponding $\widehat{\phi}_k$'s (for each $\theta_i$) are obtained from Equation~\eqref{eq:compute-phi-k}.
Then, the \(i\)-th quantized weight \(\tilde{\theta}_i\) is sampled from the set of quantization levels
\(\{\widehat{\mu}_1,\ldots,\widehat{\mu}_K\}\) according to
\begin{align}
P\!\left( \tilde{\theta}_i = \widehat{\mu}_k \mid \theta_i
\right)=
\widehat{\phi}_k(\theta_i),
\qquad k=1,\ldots,K.
\label{eq:infer-theta}
\end{align}
Compared to sampling from $\mathcal{N}(\hat{\mu}_k, \hat{\sigma}^2_k)$, posterior sampling in Equation~\eqref{eq:infer-theta} reduces memory consumption.

To enforce sparsity, we introduce a user-specified pruning parameter, the \texttt{Non-zero} rate. Each weight $\theta_i$ is associated with a score $\hat{\lambda}_i$, which reflects the likelihood of being retained. We deterministically prune by setting the $i$-th weight to zero if $\hat{\lambda}_i$ is smaller than the \texttt{Non-zero}-quantile of all $\hat{\lambda}_i$ values; otherwise, the weight is kept unchanged. Formally,
\begin{align*}
\tilde \theta_i =
0, & \qquad \text{ if } \hat{\lambda}_i < \texttt{Non-zero}\text{ quantile of all } \hat{\lambda}_i, 
\end{align*}
This deterministic rule provides exact control over the sparsity level, in contrast to stochastic pruning via posterior sampling~\citep{bai2020efficient, sun2022consistent}, which does not guarantee a fixed sparsity rate and often requires an additional pruning step.

\noindent\textbf{Bayesian averaging.}  
Given a test input $x$, the predicted output $\widehat{y}$ is computed using Bayesian averaging:
%last sampled iterate:
\begin{equation} \label{eq:infer-bam}
\widehat{y}=
    \frac{1}{M}\sum_{m=1}^Mf(x;\tilde{\theta}^{m}) 
\end{equation}
where $\tilde{\theta}^m$'s are $M$ many samples from the sparse quantized sub-distribution. 
In the following experiments, we set the $M$ default to 4.
Our ablation study (in Figure~\ref{fig:quantization_compare}) shows that Bayesian averaging consistently yields smaller accuracy degradation than the greedy alternative (detailed in Equation~\ref{eq:greedy}).

\textbf{Greedy approach} is to greedily select the most likely weight for making predictions on the test set. Specifically, for each quantized weight $\theta_i$, we choose the index $k_i^*$ corresponding to the highest posterior probability $\widehat{\phi}_{k_i^*}({\theta}_i)$. 
The quantized weight is then set to the mean $\mu_{k_i^*}$ of the selected component, and the predicted output $\widehat{y}$ is computed using these selected means. 
Formally, this greedy inference strategy is given by:
\begin{equation}\label{eq:greedy}
\widehat{y} = f(x; \hat{\mu}_{k_i^*}), \qquad k_i^* = \arg\max_{k} \widehat{\phi}_k({\theta}_i), \qquad Q(\theta_i) = \mu_{k_i^*}. 
\end{equation}
We empirically compare the greedy inference approach with Bayesian averaging in the ablation study shown in Figure~\ref{fig:quantization_compare}.

\noindent\textit{Outlier-aware windowing.} 
Recent studies show that the weight distribution of large language models (LLMs) often contains significant outliers~\citep{DBLP:conf/nips/WeiZZGZZYL22}.  
To address this, we use an outlier-aware windowing strategy to enhance the performance of \method.
Specifically, the full-precision weights are partitioned into four groups using window sizes determined by a modified interquartile range (IQR) rule~\citep{dekking2005modern}, which helps preserve large-magnitude weights during quantization.
Each group is then quantized to $K$ representative values.
As shown in the ablation study (Figure~\ref{fig:windowing}), this strategy outperforms the approach using equal-sized windows.
Implementation details are provided in Appendix~\ref{apx:implement}, and the full procedure is summarized in Algorithm~\ref{alg:training}.

\begin{algorithm}[!t]
	\caption{SQS: Variational learning and inference for sparse and quantized sub-distribution.} \label{alg:training}
	\begin{algorithmic}[1]
    \Statex \hspace{-20pt} \texttt{\color{blue}// Variational Learning}
	\Require{Training data $D$; Full-precision weights $(\theta_1,\ldots,\theta_T)$; \#components in GMM $K$; initial temperature $\tau, \tau'$; prior variance $\sigma_0^2$.}
 % \Ensure The sparse quantized weight sub-distribution $\widehat{q}(\tilde\theta)$
		\State  Initialize trainable parameters
		$ \{(\hat{\mu}_k, \hat{\pi}_k,\hat{\sigma}_k)\}_{k=1}^{K}$;
	\While{\textnormal{not converged}} 
	    \State calculate the approximate objective $\Omega_{\mathtt{apx}}$; \Comment{in Equation~\eqref{eq:approx_elbo}}
	    \State update learnable parameters with stochastic gradient descent;
	\EndWhile
\Return the sparse quantized weight sub-distribution $\widehat{q}(\tilde\theta)$.
\vspace{12pt}
\Statex \hspace{-20pt} \texttt{\color{blue}// Variational Inference}
\Require A sparse quantized weight distribution $\widehat{q}(\tilde\theta)$; \#Bayesian average $M$; percentage of non-zero weights \texttt{Non-zero} rate (\%).
% \Ensure Bayesian inferences such as $\widehat{y}$.
    \For{$m \leftarrow 1$ to $M$} 
    \State  sample quantized weight $\tilde{\theta}^{m}$ from the posterior $\widehat{q}(\tilde\theta)$; \Comment{in Equation~\eqref{eq:infer-theta}}
    \State  get pruned weight $\tilde{\theta}^m$  with \texttt{Non-zero}, according to $\hat{\lambda}_i$; 

    \EndFor
    \State predict output $\widehat{y}$ for each testing input $x$, using Bayesian average with $\{\tilde{\theta}^1,\ldots,\tilde{\theta}^M\}$; \Comment{in Equation~\eqref{eq:infer-bam}}
     \Statex\Return The set of predicted outputs.
\end{algorithmic}
\end{algorithm}

\noindent\textbf{Remarks.} DGMS~\citep{dong2022finding} adopts Gaussian mixtures, but uses them primarily as a clustering mechanism. In contrast, our method leverages a principled Bayesian framework that supports posterior inference and enables Bayesian model averaging, enhancing robustness to quantization noise. Furthermore, by unifying pruning and quantization within a spike-and-GMM variational family, our approach creates a joint optimization space that encourages globally optimal solutions across both pruning and quantization.

\subsubsection{Windowing strategy in quantization}
We observe that weight distributions vary significantly across layers, including Gaussian and long-tailed forms.  
In particular, long-tailed distributions contain a small subset of weights with large magnitudes.
Previous works~\citep{nagel2020up, hubara2021accurate, frantar2022optimal} have demonstrated that layer-wise compression methods lead to better performance. 

To address performance degradation arising from such heterogeneous distributions, we extend our proposed method to support \textit{layer-wise quantization}, where each group of weight parameters within a layer is assigned its own quantization set. This enables each layer to learn and utilize a distinct, trainable quantization set tailored to its distribution.

\textbf{Equal-size windowing.} For the equal window strategy, given a layer of weights $\theta$, we group the weights into 4 windows where each one has an equal window size $\frac{\max(\theta)-\min(\theta)}{4}$. Within each window, a $ K$-component GMM is applied to approximate the weight distribution.

\textbf{Outlier-aware windowing.} For layers with long-tailed distributions, we further introduce an \textit{outlier-aware windowing} strategy. Specifically, the weights $\theta$ in each layer are partitioned into four windows, with two dedicated to capturing the lower and upper tails of the distribution. To identify these tail regions, we apply a standard outlier detection rule based on the $5\times$ inter quartile range (IQR): let $q_1$ and $q_3$ denote the first and third quartiles of the weights, and define $\texttt{IQR} = q_3 - q_1$. The outlier-aware windows are then defined as
\begin{equation} \label{eq:outlier-window}
[\min(\theta), q_1 - 5 \times \texttt{IQR}],\qquad \text{ and }\qquad [q_3 + 5 \times \texttt{IQR}, \max(\theta)]
\end{equation}
 Within each of the four windows in every layer, we fit a $K$-component GMM to approximate the local weight distribution. 

We adopt the layer-wise quantization scheme with outlier-aware windowing in all our experiments. This approach improves the preservation of extreme values during quantization and enhances robustness across layers. An ablation study evaluating the effectiveness of outlier-aware windowing is presented in Figure~\ref{fig:windowing}.

\subsection{Theoretical Justification of SQS}\label{sec:theory}

For clarity, this section focuses on regression tasks with fully connected neural networks. We analyze the variational posterior of sparse and quantized neural networks,  i.e., the optimization of Equation~\eqref{eq:exact_ELBO}.  We show that this variational posterior converges to a true regression function under some mild conditions. 

Consider a regression problem with random covariates, 
\begin{equation}\label{eq:regression}
    Y_i=f_0(X_i)+\varepsilon_i, \quad \text{ for } i=1,\ldots, n,
\end{equation}
where $f_0: [0, 1]^p \rightarrow \mathbb{R}$ is the underlying unknown true function, $X_i \sim \mathcal{U}([0, 1]^p)$ is sampled from a $p$-dimensional uniform distribution, $\varepsilon_i\overset{iid}{\sim}\mathcal{N}(0,\sigma^2_{\varepsilon})$ is the noise term from a Gaussian distribution of zero mean and variance $\sigma^2_{\varepsilon}$. Let $P_0$ denote the true underlying probability measure of the data, and $p_0$ denote the corresponding density function. An $L$-hidden-layer fully connected NN with constant layer width $N $ and parameters $W_{\ell} \in \mathbb{R}^{N \times N}$, $b_{\ell} \in \mathbb{R}^{N}$ and activation function $\sigma(\cdot)$ can be defined as: 
\begin{equation}\label{eq:model}
     f_{\theta}(X)=W_{L+1}\sigma_{b_{L}}(W_{L}\sigma_{b_{L-1}}\ldots \sigma_{b_1}(W_1X)) + b_{L+1}.
\end{equation}

For simplicity, $\sigma_{ \varepsilon}$ is assumed to be known. Let $s^*$ be the ``oracle'' sparsity level (see Equation~\ref{eq:def_sstar} in Appendix \ref{apx:converge-obj} for formal definition) and $H(T, s, K)$ be the set of network weight parameters such that the network has a sparsity of $s$ and shares at most $K$ distinct values. Let $P_0$ and $P_\theta$ be the true data distribution and the distribution under parameter $\theta$, respectively.

Theorem~\ref{thm:convergence} carries a proof sketch stating the two-step structure: Lemma~\ref{lem:bound_elbo} upper-bounds the variational objective with high probability; Lemma~\ref{lem:lmbound1} converts that bound into convergence of the variational posterior in squared Hellinger distance.

\begin{lemma}\label{lem:bound_elbo}
 Under Conditions~\ref{cond:constant_width}-\ref{cond:sigma}, with high probability,
\begin{equation*}
    \inf_{q(\theta)\in \mathcal{F}}\Big\{\KL\left(q(\theta)\|\pi(\theta|\lambda)\right)+\int_{\R^T}l_n(P_0, P_\theta)q(\theta)\mathit{d}\theta\Big\} \leq Cn(r_n^*+\xi_n^*)
\end{equation*}
where $C$ is either some positive constant if $\lim n(r_n^*+\xi^*_n)=\infty$, or any diverging sequence if $\lim\sup n(r_n^*+\xi^*_n)\neq \infty$. And $l_n(P_0, P_\theta)$ is defined as:
\begin{equation*}
    l_n(P_0, P_\theta) = \frac{1}{2\sigma^2_{ \varepsilon}}(\norm{Y - f_{\theta}(X)}^2_2 - \norm{Y - f_{0}(X)}^2_2).
\end{equation*}
\end{lemma}

For any $\delta > 1$, let $s_n := s^*\log^{2\delta-1}(n)$.

\begin{lemma}{\label{lem:lmbound1}}
Under Conditions~\ref{cond:constant_width}-\ref{cond:bound}, if $\sigma_0^2$ is set to be constant and $\lambda\leq T^{-1}\exp\{-M nr_n^*/s_n\}$ for any positive diverging sequence $M\rightarrow\infty$, then with high probability, then we have
\begin{equation}\label{eqbound1}
\int d^2(P_{\theta}, P_0)\widehat{q}(\theta)\mathit{d}\theta \leq C\varepsilon^{*2}_n + \frac{3}{n}\inf_{q(\theta) \in \mathcal F}\Bigl\{ \KL(q(\theta)\|\pi(\theta|\lambda))
+ \int l_n(P_0, P_{\theta})q(\theta)\mathit{d}\theta \Bigr\},
\end{equation}
where $C$ is some constant, and $$
\varepsilon^*_n :=\varepsilon_n(L,N,s^*)=\sqrt{r_n(L,N,s^*)}\log^\delta(n), \mbox{ for any } \delta > 1.
$$
\end{lemma}
We refer to Appendix~\ref{sec:proof-bound-elbo} for the proof of Lemma~\ref{lem:bound_elbo} and Appendix~\ref{sec:proof-lmbound1} for the proof of Lemma~\ref{lem:lmbound1}. 
\begin{theorem}\label{thm:convergence}
Let   $r_n^* = \left((L+1)s^*/n\right)\log N + (s^*/n \log(p\sqrt{n/s^*}))$,
 $\varepsilon_n^*=\sqrt{r_n^*}\log^\delta(n)$ for any $\delta > 1$ from Lemma~\ref{lem:lmbound1}, and 
 $ \xi_n^* = \inf_{\theta \in H(T, s, K), \norm{\theta}_{\infty}\leq B} \norm{f_{\theta}-f_0}_{\infty}^2$. Then, under mild conditions specified in the supplementary material, with high probability:
    \begin{equation}
        \int_{\R^{T}} d^2(P_{\theta}, P_0) \widehat{q}(\theta)d\theta \leq C \varepsilon_n^{*2} + C'(r_n^*+\xi_n^*),
    \end{equation}
where $d(\cdot,\cdot)$ denotes the Hellinger distance, and $C$ and $C'$ are some constants. 

\begin{proof}[Sketch of Proof]
Based on prior work~\citep{bai2020efficient}, the proof proceeds in two steps.  
Lemma~\ref{lem:bound_elbo} establishes a high-probability bound on the ELBO in Equation~\eqref{eq:exact_ELBO}.  
Lemma~\ref{lem:lmbound1} connects this bound to the convergence of the variational distribution toward the true full-precision posterior.  
Together, these results show that the variational posterior induced by our method converges to the true regression function with high probability.  
The full proof is in Appendix~\ref{apx:converge-obj}.
\end{proof}
\end{theorem}

\textbf{Remark.}
Similar to previous Bayesian sparse DNN results~\citep{bai2020efficient,cherief2020convergence}, the convergence rate of variational Bayes is determined by the deep neural network structure via 1) statistical estimation error $\varepsilon_n^*$, 2) variational error $r_n^*$, and 3) approximation error $\xi_n^*$. The first two are positively related to the network capacity, while the third one is negatively related to the network capacity. The estimation error $\varepsilon_n^*$ and variational error $r_n^*$ vanish as $n \to \infty$. Prior work~\citep{beknazaryan2022function} shows that under $B \geq 2$, $K \geq 6$, and $\beta$-H\"older smoothness of $f_0$, the approximation error $\xi_n^*$ also vanishes.

While the theoretical analysis mainly considers an $L$-hidden-layer fully connected NN with constant layer width $N$, our method \method is empirically validated on a variety of models such as ResNets, BERT-based models, and LLMs (refer to Section \ref{sec:experiments}). 
\section{Related Work}

\textbf{Weight pruning} was initially introduced by~\citet{lecun1989optimal}, with further development by \citet{hassibi1993optimal} through a mathematical method known as the Optimal Brain Surgeon (OBS). This approach selects weights for removal from a trained neural network using second-order information. 
Subsequent improvements, as indicated by studies~\citep{dong2017learning, wang2019eigendamage, singh2020woodfisher}, have adapted OBS for large-scale DNNs by employing numerical techniques to estimate the second-order information required by OBS.  
Meanwhile, \citet{louizos2018learning} introduced an $L_0$-regularized method to promote sparsity in DNNs. \citet{frankle2018the} established a critical insight that within a randomly initialized DNN, an optimal sub-network can be identified and extracted. 
Recently, \citet{xia2024sheared} showed that structured pruning combined with targeted retraining can significantly reduce computational costs while preserving robust performance for large language models.
Concurrently, spike-and-slab distributions have been employed to promote sparsity in DNNs using Bayesian Neural Networks formulation~\citep{deng2019adaptive, blundell2015weight, bai2020efficient}.

\noindent\textbf{Low-bit quantization.} Quantization improves DNN efficiency, particularly in resource-constrained environments~\citep{sze2017efficient, frantar2022gptq, lin2024awqactivationawareweightquantization,DBLP:conf/aistats/LinGJBP25}. Research in this field typically follows two paradigms: discontinuous-mapping and continuous-mapping quantization.
Discontinuous-mapping methods project full-precision weights onto a low-bit grid using rounding operations~\citep{gupta2015deep, hubara2018quantized, wu2018training, louizos2018relaxed, courbariaux2015binaryconnect, de2018high, marchesi1993fast}. The non-differentiability of these mappings necessitates the use of the \textit{straight-through estimator} (STE) for gradient approximation~\citep{Courbariaux2016BinaryNet, rastegari2016xnor}. However, STE-based training may introduce pseudo-gradients, leading to training instability~\citep{yin2018understanding}.
Meanwhile, many researchers propose post-training quantization methods that have limited access to the training dataset~\citep{wang2020towards, hubara2021accurate, li2021brecq, frantar2022optimal, frantar2022gptq, lin2024awqactivationawareweightquantization}.

Continuous-mapping quantization offers an alternative that avoids pseudo-gradients, leading to more stable training~\citep{yin2018understanding,DBLP:conf/acl/NielsenSG25}.
These methods often use variational learning~\citep{ullrich2017soft, louizos2017bayesian, shayer2018learning} or Markov Chain Monte Carlo techniques~\citep{roth2018bayesian} to approximate discrete weight distributions.
However, variational methods often require manual prior specification~\citep{ullrich2017soft, louizos2017bayesian, shayer2018learning}, while MCMC approaches can be memory-intensive~\citep{roth2018bayesian}.
DGMS~\citep{dong2022finding} addresses these limitations through automated quantization using GMMs. Our work extends DGMS by integrating pruning and quantization into a unified framework, thereby achieving higher compression rates.

\noindent\textbf{Joint pruning and quantization.}
A growing line of work optimizes sparsity and precision jointly: Bayesian formulations derive both from a single prior or posterior~\citep{ullrich2017soft, louizos2017bayesian, achterhold2018variational}, while non-Bayesian approaches rely on differentiable gates, second-order saliency, or joint policy search~\citep{wang2020differentiable, frantar2022optimal, wang2020apq, bai2023unified}.
Bayesian Bits~\citep{van2020bayesian} unifies pruning and quantization by gating a chain of residual terms that double the bit-width, with pruning as the $0$-bit case. Unlike its gates on a \emph{uniform} grid, \method learns the levels themselves as GMM means with a per-weight retention probability, giving a non-uniform codebook, exact sparsity control, and a posterior for Bayesian averaging.

\textbf{Large language model compression.} Recent work on compressing large models has pursued several complementary directions. % Variational Network Quantization (VQN)~\cite{achterhold2018variational} is a Bayesian compression approach that combines a multimodal quantization prior with variational dropout to jointly prune and ternary-quantize weights.
SpinQuant~\citep{DBLP:conf/iclr/Liu0FSCKCTB25} applies learned rotations to weights and embeddings to reduce outliers, making models easier to quantize. LeanQuant~\citep{DBLP:conf/iclr/0011S25} introduces a loss-aware post-training quantization method that learns adaptive affine transformations and non-uniform quantization grids, aiming to preserve outlier-sensitive weights. EfficientXpert~\citep{zhao2025efficientxpert} uses LoRA-guided pruning to obtain domain-aware pruned models. ~\citet{xu2025towards} targeted vision--language models and adaptively pruned attention heads using entropy-based effective rank and the Kolmogorov--Smirnov distance, reporting substantial FLOP reductions.  In contrast, our \method introduces a learnable codebook and a spike-and-slab posterior tailored to sparse, quantized weights in the low-precision regime.

\section{Experiments}
\label{sec:experiments}
In this section, we show that our \method achieves a much higher compression rate (see the second-to-last column in Tables~\ref{tab:cifar10}-\ref{tab:sst}) while incurring a comparable or smaller performance drop (see the last column in the same tables). Through ablation studies, we further validate that (1) under the same sparsity level, the spike-and-slab prior more effectively preserves model accuracy (see Table~\ref{tab:prior_compare}). (2) Under identical hyperparameter settings, we show that \method with Bayesian averaging outperforms the greedy approach (see Figure~\ref{fig:quantization_compare}).

\subsection{Experiment settings}
We evaluate all methods using two metrics: the \emph{compression rate} and the \emph{performance drop} (i.e., accuracy drop or F1 score drop).  
The compression rate is defined as the memory footprint of the compressed model over the original dense full-precision model:  
\begin{align} \label{eq:compression-rate}
\text{Compression rate}=\frac{32\times \texttt{original weight counts}}{\log_2 K\times \texttt{nonzero weight counts} + 32\times K},
\end{align}
where $K$ is the codebook size.
\texttt{Non-zero rate}  is the percentage of weights that are pruned to zero. In all our experiments, the \texttt {Non-zero rate} is configured as a hyperparameter to control the sparsity.
Equation~\eqref{eq:compression-rate} accounts for the codebook and the quantized value indices assigned to the nonzero weights, but, following the convention adopted by the pruning baselines we compare against (e.g., L-OBS~\citep{dong2017learning}, PLATON~\citep{zhang2022platon} and ExactOBS/OBC~\citep{frantar2022optimal}), it does not separately charge bits for the binary sparsity mask that records \emph{which} weights are pruned. We adopt this convention so that the reported compression rates in Tables~\ref{tab:cifar10}--\ref{tab:sst} remain directly comparable to the values reported for these baselines in their original papers, all of which similarly define compression as a function of remaining/quantized parameter count alone. We note that some deployment-oriented compression pipelines, such as \citep{han2015deep,DBLP:conf/iclr/DettmersSEKFABH24}, instead charge explicit bits for the sparsity structure (e.g., via a bitmap or compressed-sparse-row index); under that stricter accounting, absolute compression rates for all pruning-based methods, including ours, would be lower, though the relative ranking among methods is unaffected.
In practice, for the models we study, the product  ``$\log_2 K\times$ nonzero weight counts'' is usually much larger than  $ 32\times K$, 
because the number of nonzero weights is very large and the codebook size is small. In this regime, the memory used to store the indices dominates, and the memory used by the codebook is very small.

We compare methods of different compression types (the ``Compression type'' column): ``P+Q'' denotes combined pruning and quantization, ``P'' denotes pruning only, and ``Q'' denotes quantization only.

In our experiments, we store the indices of the weights using \texttt{INT4}, and we perform computation in \texttt{FP32}. This setting is common in prior work, such as QLoRA~\citep{DBLP:conf/nips/DettmersPHZ23} and learned codebook methods~\cite{van2017neural,dong2022finding}. 

To ensure a fair comparison, each method is initialized with the same full-precision pre-trained model and is run with the same set of hyperparameters for compression.  All methods are constrained to a maximum runtime of 24 hours. The resulting compressed models are then evaluated on the same test sets, and the key performance metrics are summarized in the corresponding tables.
Appendix~\ref{apx:exp-set} provides detailed experimental configurations and baseline settings.

\subsection{Experimental analysis}

\textbf{Compression on ResNet models.} Table \ref{tab:cifar10} summarizes the result of all methods for compressing 
ResNet-20, ResNet-32 and ResNet-56, evaluated on the CIFAR-10 dataset.
On compressing the  ResNet-20 model, our \method attains a better compression rate than the baselines.
On compressing ResNet-32 and ResNet-56 models, our \method attains substantially higher compression rates while incurring smaller accuracy drops compared to the baselines.
Optimizing pruning and quantization separately overlooks redundancies in each step; by merging them into a single optimization, we effectively eliminate these inefficiencies.

\begin{table*}[!t]
    \centering
       \caption{For compressing ResNet models, we benchmark all methods evaluated on the CIFAR-10 dataset. Using ResNet-32 and ResNet-56 models, our \method consistently achieves higher compression rates with smaller Top-1 accuracy drops compared to all baselines.}
 
    \label{tab:cifar10}
    \resizebox{\linewidth}{!}{%
    \begin{tabular}{lrrrrrc}
    \toprule
    \multirow{5}{*}{\rotatebox[origin=c]{90}{ResNet-20}} & Methods & Compression  & Bits & Non-zero rate & Compression   & Top-1 accuracy   \\
    & & type &  &  (\%)  & rate & drop \\
    \cline{2-7}
    & LQNets~\citep{zhang2018lq}& Q & $2$ & $100\%$ & $16\times$ & $1.20\%$ \\
    & DGMS~\citep{dong2022finding} & P+Q & $2$ & $56\%$ & $29\times$ & $\mathbf{0.87\%}$\\
     & \method (Ours) & P+Q & $2$ & $\mathbf{50\%}$ & $\mathbf{32\times}$ & $1.47\%$ \\
    \midrule
     \multicolumn{7}{c}{\textbf{(a)} Compressing 32Bits ResNet-20 model on CIFAR-10 dataset with Top-1 accuracy $92.60\%$.} \\
     \midrule
    \multirow{5}{*}{\rotatebox[origin=c]{90}{ResNet-32}}  & Method & Compression  & Bits & Non-zero rate & Compression  & Top-1 accuracy \\
    & & type &  &  (\%)  & rate & drop \\
    \cline{2-7} 
    % \cmidrule(lr){2-7} 
    & TTQ~\citep{zhu2017trained} & Q & $2$ & $100\%$ & $16\times$ & $1.90\%$ \\
    & DGMS~\citep{dong2022finding} & P+Q & $2$ & $59\%$ & $27\times$ & $1.30\%$ \\
    & \method (Ours) & P+Q & $2$ & $\mathbf{50\%}$  & \textbf{32$\times$} & $\mathbf{1.29\%}$ \\
    \midrule
    \multicolumn{7}{c}{\textbf{(b)} Compressing 32Bits ResNet-32 model on CIFAR-10 dataset with Top-1 accuracy $93.53\%$.} \\
    \midrule
    \multirow{6}{*}{\rotatebox[origin=c]{90}{ResNet-56}} & Method & Compression  & Bits &  Non-zero  rate & Compression   & Top-1 accuracy  \\
    & & type &  &  (\%)  & rate & drop \\
    \cline{2-7} %
    % \cmidrule(lr){2-7}
    & TTQ~\citep{zhu2017trained} & Q & 2 & $100\%$ & 16$\times$ & $1.06\%$ \\ 
    & L1~\citep{li2017pruning} & P & 32 & $10\%$ & 10$\times$ & $1.83\%$ \\
    & DGMS~\citep{dong2022finding} & P+Q & $2$ & $52\%$ & $31\times$ & $0.89\%$ \\
     & \method (Ours) & P+Q & 2 & $\mathbf{50\%}$ & $\mathbf{32\times}$ & $\mathbf{0.84\%}$  \\
    \bottomrule
     \multicolumn{7}{c}{\textbf{(c)} Compressing 32Bits ResNet-56 model on CIFAR-10 dataset with Top-1 accuracy $94.37\%$.}
    \end{tabular}%
    }
\end{table*}

\begin{table*}[!ht]
    \centering
         \caption{Compressing 32Bits BERT-base model on SQuADv1.1 dataset with F1 score $88.68\%$. Our \method achieves higher compression rates with smaller F1 score drops compared to all baselines.}
    \label{tab:bert_squad}
    \resizebox{\linewidth}{!}{%
    \begin{tabular}{crrrrrr}
    \toprule
   \multirow{9}{*}{\rotatebox[origin=c]{90}{BERT-base}} &Methods & Compression & Bits  &  Non-zero rate & Compression   &F1 score  \\
   & & type &  &  (\%)  & rate & drop \\
     % \cline{2-7} %
    \cmidrule(lr){2-7}
    & GMP~\citep{zhu2017prune} & P & $32$ & $50\%$ & $2\times$ & $22.89$ \\
    & L-OBS~\citep{dong2017learning} & P & $32$ & $50\%$ & $2\times$ & $10.86$ \\
    & ExactOBS~\citep{frantar2022optimal} & P & $32$ & $25\%$ & $4\times$ & $6.43$ \\
    & PLATON~\citep{zhang2022platon} & P & $32$ & $20\%$ & $5\times$ & $2.20$ \\
    & OBQ~\citep{frantar2022optimal} & Q & $3$ & $100\%$ & $11\times$ & $3.24$ \\ 
    & GPTQ~\citep{frantar2022gptq} & Q & $3$ & $100\%$ & $11\times$ & $2.51$ \\
    & OBC~\citep{frantar2022optimal} & P+Q & $4$ & $50\%$ & $16\times$  & $2.33$  \\
     & \method (Ours) & P+Q & 4 & $\mathbf{25\%}$ & $\mathbf{32\times}$ & $\mathbf{1.66}$ \\
    \bottomrule
    \end{tabular}%
    }

\end{table*}

\noindent\textbf{Compression on BERT-base model.}
We apply our compression method to the BERT-base model~\citep{devlin2018bert} and evaluate its performance on the SQuAD v1.1 dataset~\citep{rajpurkar2016squad}. 
The evaluation metrics include the F1 score drop and the compression rate. 
As shown in Table~\ref{tab:bert_squad}, our method achieves the lowest F1 score drop and the highest compression rate, outperforming existing methods. This
 demonstrates the effectiveness of our \method method in preserving accuracy under aggressive compression.

\begin{table*}[!t]
    \centering
   \caption{Compression results for Llama3.2 and Qwen2.5 models on the SST-2 dataset. Our \method achieves significantly higher compression rates than AWQ while maintaining comparable ($\le 3\%$) performance drops.}
    \label{tab:sst}
        \resizebox{\linewidth}{!}{%
    \begin{tabular}{lrrrrrr}
    \toprule 
     \multirow{5}{*}{\rotatebox[origin=c]{90}{Llama3.2}} & Method &  Compression  & Bits & Non-zero rate  & Compression  & Top-1 accuracy \\
     & & type &  &  (\%)  & rate & drop \\
   % \cline{2-7} % 
   \cmidrule(lr){2-7}
    & AWQ~\citep{lin2024awqactivationawareweightquantization} & Q & $4$ & $100\%$ & $8\times$ & $\mathbf{0.46\%}$ \\
     & DGMS~\citep{dong2022finding} & P+Q & $6$ & $82\%$ & $7\times$ & $46.67\%$ \\
    & \method (Ours) & P+Q & $6$ & $\mathbf{25\%}$ & $\mathbf{21\times}$ & $1.48\%$ \\ 
    \midrule
    \multicolumn{7}{c}{\textbf{(a)} Compressing 32Bits Llama3.2-1B model on SST-2 dataset with Top-1 accuracy $94.72\%$.} \\
    \midrule
   \multirow{5}{*}{\rotatebox[origin=c]{90}{Qwen2.5}}  & Method &  Compression  & Bits & Non-zero rate & Compression   & Top-1 accuracy \\
   & & type &  &  (\%)  & rate &  drop \\
    % \cline{2-7} %
    \cmidrule(lr){2-7}
    &  AWQ~\citep{lin2024awqactivationawareweightquantization} & Q &  4 & 100\% & $8\times$ & $\mathbf{1.54\%}$ \\
    &  DGMS~\citep{dong2022finding} & P+Q & $6$ & $\mathbf{34\%}$ & $\mathbf{16}\times$ & $50.80\%$ \\
    & \method (Ours) & P+Q & $6$ & $50\%$ & ${11\times}$ & $2.46\%$\\
    \bottomrule
    \multicolumn{7}{c}{\textbf{(b)} Compressing 32Bits Qwen2.5-0.5B model on SST-2 dataset with Top-1 accuracy $92.60\%$.}
    \end{tabular}%
}
\end{table*}

\noindent\textbf{Compression on Llama and Qwen models.} In Table~\ref{tab:sst}, we compare our \method with others on the SST-2 task in the GLUE benchmark using Llama3.2-1B and Qwen2.5-0.5B models, as our method could preserve the weight outliers, which are crucial in maintaining the performance~\citep{lin2024awqactivationawareweightquantization}. 
We further observe that DGMS~\citep{dong2022finding} incurs a large performance drop for compressing Llama3.2-1B and Qwen2.5-0.5B, which occurs because the weight distribution of the self-attention layer is not Gaussian (see Figure~\ref{fig:windowing}(left)), and it fails to capture large magnitude weights. Furthermore, DGMS does not allow customization of the sparsity level; thus, it presents an unreasonable performance drop.

\begin{table*}[!t]
    \centering
     \caption{Impact of the Gaussian prior and the spike-and-slab prior, for compressing a 32 bits ResNet-18 model on the CIFAR-100 dataset with Top-1 Accuracy $79.26\%$. The spike-and-slab prior used in our \method consistently yields better performance than the Gaussian prior across all sparsity-level settings.}
    \label{tab:prior_compare}
    
    \begin{tabular}{lccr|rr}
    % \hline
    \toprule
   \multirow{6}{*}{\rotatebox[origin=c]{90}{ResNet-18}} &\multirow{2}{*}{Bits}  & \multirow{2}{*}{Non-zero rate (\%) } & \multirow{2}{*}{Compression rate} & \multicolumn{2}{c}{Top-1 accuracy drop}\\
       &&& &Gaussian prior  & Spike-and-slab prior (Ours) \\
    % \cline{2-6}  
    \cmidrule(lr){2-6}
    & $4$ & $50\%$ & $16\times$ & $4.51\%$ & $\mathbf{3.12}\%$ \\
    & $4$ & $40\%$ & $20\times$ & $5.60\%$ & $\mathbf{3.21}\%$ \\
    & $4$ & $30\%$ & $27\times$ & $11.42\%$ & $\mathbf{5.54}\%$ \\
    & $4$ & $20\%$ & $40\times$ & $44.04\%$  & $\mathbf{5.59}\%$ \\
    % \hline
    \bottomrule
    \end{tabular}
\end{table*}

\begin{figure}[!t]
    \centering
    \includegraphics[width=.92\linewidth]{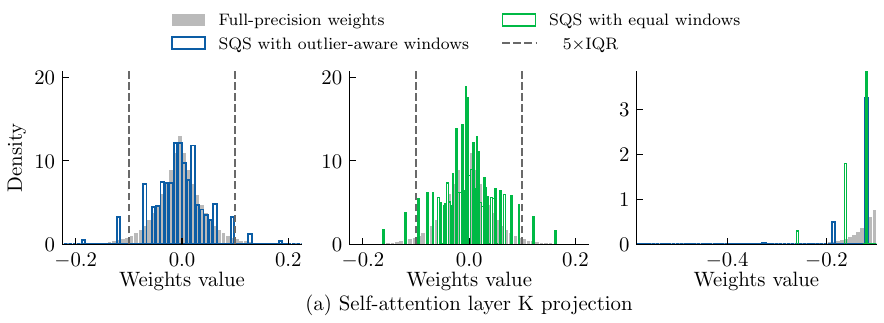}
    \includegraphics[width=.92\linewidth]{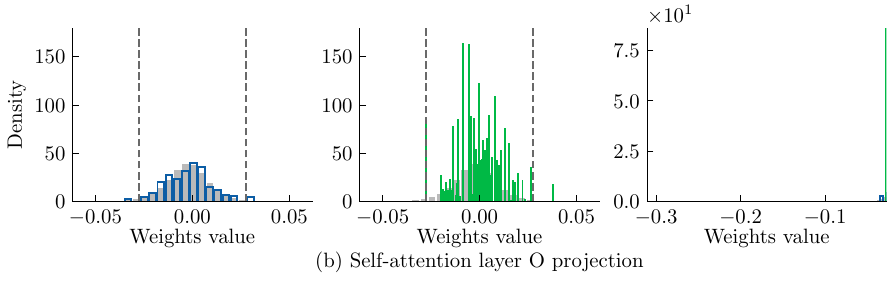}
    \includegraphics[width=.92\linewidth]{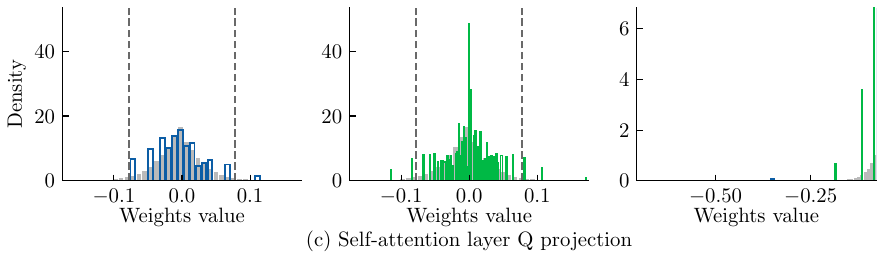}
    \includegraphics[width=.92\linewidth]{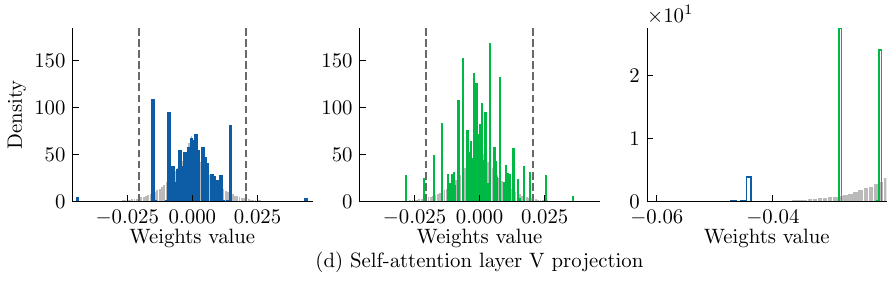}
    \caption{For the compressed weight distributions of the $K$, $O$, $Q$, and $V$ matrices in the self-attention layer of the Llama3.2-1B model, \method using the outlier-aware window strategy (\textbf{Left}) more effectively preserves the characteristics of the full-precision weight distribution compared to the equal-sized window strategy (\textbf{Middle}). This improvement is particularly noticeable in the left tail region, as highlighted in (\textbf{Right}).}
    \label{fig:windowing}
\end{figure}

\subsection{Ablation studies}\label{sec:ablation}

\noindent\textbf{Prior selection: Gaussian vs.\ spike-and-slab.}
We evaluate how the choice of prior affects compression performance, comparing a Gaussian prior (Appendix~\ref{apx:gaussian_prior}) to the spike-and-slab prior in Equation~\eqref{eq:marginal_prior}. Specifically, we compress ResNet-18 at varying sparsity levels by representing each layer’s weights with $K=16$ components, and evaluate accuracy on CIFAR-100.

As shown in Table~\ref{tab:prior_compare}, the spike-and-slab prior consistently outperforms the Gaussian prior across all sparsity levels. The gap becomes particularly pronounced at high sparsity, where the Gaussian prior suffers substantial degradation, suggesting that it may be less effective at inducing posterior sparsity in DNN weights. We leave a deeper investigation of this behavior to future work.

\noindent\textbf{Windowing strategy: equal-size vs. outlier-aware window.}
We use the first-layer attention weights of Llama3.2-1B as a case study. The full-precision weights exhibit a pronounced long-tail distribution, where a small fraction of entries have large magnitudes. Additional statistics on layer-wise weight distributions are reported in Appendix~\ref{apx:extended_result}. Motivated by this observation, our \method uses an outlier-aware window strategy to better fit the full-precision weight distribution in Equation~\eqref{eq:outlier-window}. Figure~\ref{fig:windowing} visualizes the resulting quantized weights under the equal-window and outlier-aware window strategies.

% ======== AI-EDITED START (assistant) ========
\begin{table}[!ht]
    \centering
  \caption{> Comparison of outlier-aware and equal-size windowing for SQS compression of Qwen2.5-0.5B on SST-2, with a full-precision Top-1 accuracy of 92.60\%. At 6 bits and 50\% nonzero weights, the accuracy drop is 2.46 percentage points with outlier-aware windowing and 5.40 percentage points with equal-size windowing, a difference of 2.94 percentage points.}
    \begin{tabular}{lllccc}
    \toprule
       \multirow{3}{*}{\rotatebox[origin=c]{90}{Qwen2.5}} &Method& Windowing strategy & Bits & Non-zero rate (\%)  & Top-1  accuracy drop $\downarrow$ \\
    \cmidrule(lr){2-6}
       &\multirow{2}{*}{SQS} 
         & Outlier-aware window & $6$ & $50\%$ & $\mathbf{2.46\%}$ \\
         & & Equal-size window          & $6$ & $50\%$  & $5.40\%$ \\

    \bottomrule
    \end{tabular}
    \label{tab:iqr-equal}
    
\end{table}

As shown in Table 5, outlier-aware windowing reduces the Top-1 accuracy drop on Qwen2.5-0.5B from 5.40 to 2.46 percentage points at the same bit width and nonzero rate.  The equal-size strategy spreads its windows uniformly over the weight range,
so the few large-magnitude entries in the long tail fall into wide windows and are coarsely
quantized, which is precisely the tail region that Figure~\ref{fig:windowing} shows to be
distorted. Preserving these outlier weights is therefore a primary driver of the accuracy that
\method retains on heavy-tailed LLM weights.

As shown in Figure~\ref{fig:windowing} (left, middle), the quantized weights produced by \method under the outlier-aware window strategy more closely match the full-precision distribution than those obtained with an equal-window strategy. Figure~\ref{fig:windowing} (right) further highlights that the outlier-aware window improves fidelity in the tails, capturing extreme-magnitude weights more accurately than equal windowing.

\noindent\textbf{Inference strategies: Bayesian averaging vs. greedy approach.} 
We evaluate the effectiveness of two inference strategies within our \method framework: (1) Bayesian averaging as defined in Equation~\eqref{eq:infer-bam}, and (2) greedy approach as defined in Equation~\eqref{eq:greedy}. To ensure a fair comparison, we assess the performance of compressing ResNet-18 and ResNet-50 models while varying the number of Gaussian components. The sparsity level is fixed to zero (i.e., no pruning), so that all performance degradation arises purely from quantization.

As shown in Figure~\ref{fig:quantization_compare}, using fewer components results in a larger accuracy drop. Under the same number of components, \method with Bayesian averaging consistently achieves a smaller accuracy drop compared to the greedy approach.

\begin{figure}[!t]
\centering
    \includegraphics[width=\linewidth]{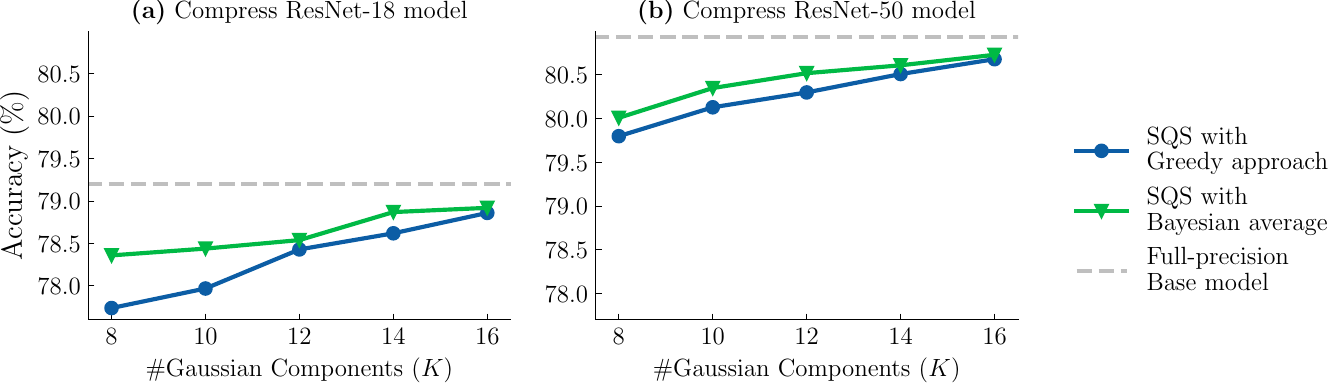} 
  \caption{Comparison of inference accuracy on the CIFAR-100 dataset using ResNet-18 (left) and ResNet-50 (right). 
    Under the same number of Gaussian components, \method with Bayesian averaging 
    (in Equation~\ref{eq:infer-bam}) results in a smaller accuracy drop compared to using a greedy approach (in Equation~\ref{eq:greedy}).}
    \label{fig:quantization_compare}
\end{figure}

\noindent\textbf{Effect of the number of Bayesian-averaging samples.} As shown in Table~\ref{tab:bm}, the number of averaged posterior samples $M$
directly affects the robustness of the compressed model: a single posterior sample ($M{=}1$)
incurs a $3.76\%$ Top-1 accuracy drop, whereas averaging over $M{=}50$ samples reduces the drop
to $2.63\%$. Most of this benefit is
realized with only a few samples, as $M{=}5$ already attains a $2.86\%$ drop and further
increasing $M$ yields only marginal improvement. This indicates that a small number of posterior
samples suffices in practice, and that Bayesian averaging mainly improves robustness to
quantization noise rather than providing a large accuracy gain.

\begin{table}[!t]
    \centering
     
    \caption{Impact of the number of Bayesian-averaging samples $M$ (Equation~\ref{eq:infer-bam}). Compressing ResNet-18 on
    CIFAR-100 with $K{=}16$ components at $50\%$ non-zero weights. Accuracy drop is measured
    against the full-precision model ($79.26\%$). Increasing $M$ steadily reduces the drop.}
    \label{tab:bm}
    \begin{tabular}{l|cccc}
    \toprule
       Inference strategy   &   $M{=}1$ & $M{=}5$ & $M{=}20$ & $M{=}50$ \\
    \midrule
       Top-1 accuracy\ Drop (\%) $\downarrow$ & $3.76$ & $2.86$ & $2.85$ & $\mathbf{2.63}$ \\
    \bottomrule
    \end{tabular}
    
\end{table}

\begin{table}[!t]
    \centering
    \caption{{Comparison of \method with Bayesian Bits on ResNet-56/CIFAR-10. The full-precision model achieves $94.37\%$ Top-1 accuracy.}}
    \label{tab:bayesian-bits}
    \resizebox{\linewidth}{!}{%
    \begin{tabular}{lccrrcc}
    \toprule
    Method & Compression  & Weight bits & Non-zero  & Effective & Compression  & Top-1  accuracy\\
    &type& &rate (\%)& bits/weight& rate& drop \\
    \midrule
    Bayesian Bits & P+Q & $2/4/8$ (mixed) & $100\%$ & $\approx 2.36$ & $\approx 13.5\times$ & $6.79\%$ \\
    \method (Ours) & P+Q & $2$ & $\mathbf{50\%}$ & $\mathbf{1.87}$ & $\mathbf{17.0\times}$ & $\mathbf{0.84\%}$ \\
    \bottomrule
    \end{tabular}%
}
\end{table}

\noindent\textbf{Comparison with Bayesian Bits.} 
In Table~\ref{tab:bayesian-bits}, \emph{effective bits per weight} is the average storage cost per weight of the original dense model. We define $b_{\mathrm{eff}}=B_{\mathrm{comp}}/T$, where $T$ is the original dense model's weight count and $B_{\mathrm{comp}}$ is the total bit cost counted for the compressed weight representation in this comparison. Relative to an FP32 baseline, the corresponding compression rate is $32T/B_{\mathrm{comp}}=32/b_{\mathrm{eff}}$.

We use the official Bayesian Bits implementation released by Qualcomm AI Research\footnote{\url{https://github.com/Qualcomm-AI-research/BayesianBits}} and follow the published CIFAR-10 recipe without modification. We use a batch size of $128$ and Adam to optimize the weights, quantization, and gate parameters. Bayesian Bits is trained from scratch for $70$ epochs for $2$ hours, whereas \method is initialized from a pre-trained model and trained for $11$ epochs.

We compare Bayesian Bits with \method on ResNet-56/CIFAR-10, for which the full-precision model achieves $94.37\%$ Top-1 accuracy.
As shown in Table~\ref{tab:bayesian-bits}, \method achieves higher Top-1 accuracy and greater compression with fewer effective bits per weight.

\section{Conclusion}
In this paper, we proposed a unified framework for compressing full-precision DNNs by combining pruning and quantization into one integrated optimization process through variational learning. Unlike conventional approaches that apply pruning and quantization sequentially—often resulting in suboptimal solutions—our method jointly explores a broader solution space, achieving significantly higher compression rates with comparable performance degradation.
To address the intractability of the original objective, we introduce an efficient approximation that enables scalable optimization. We evaluate our method across a range of benchmarks, including ResNets, BERT-base, Llama3.2, and Qwen2.5. 
Experimental results demonstrate that our approach consistently outperforms existing baselines in compression rate while maintaining competitive accuracy, highlighting its potential for efficient deployment in resource-constrained environments.

\subsection*{Broader Impact}

Model compression can reduce memory use and may lower the cost of deploying DNNs on resource-constrained hardware. These benefits should be weighed against training-to-deployment trade-offs. Our method requires additional optimization after pretraining, which can increase training time and energy use before deployment. Inference also involves a design choice: Bayesian averaging can improve robustness to quantization noise, but using multiple posterior samples may increase latency compared with a single compressed model. Practical deployments should therefore measure end-to-end latency, memory use, energy consumption, and accuracy under the target hardware and workload.

SQS does not change the task or data distribution of the original model, so its safety risks largely inherit those of the base model. Compression may make models easier to deploy more widely, including in settings where monitoring is limited. For applications involving sensitive, high-stakes, or user-facing decisions, compressed models should be evaluated for robustness, fairness, privacy leakage, and failure modes after compression, rather than assuming that performance on aggregate benchmarks is sufficient.

\subsection*{Limitations}

SQS is a compression-aware training method that optimizes its objective using task-specific training data. In our LLM experiments, the base Llama3.2-1B and Qwen2.5-0.5B models are first fine-tuned on SST-2 before compression; omitting this task-adaptation step leads to substantial performance degradation. Consequently, the reported results characterize the compression of task-adapted models rather than the preservation of the models' general-purpose capabilities. Evaluating whether SQS maintains performance on unrelated tasks or under distribution shifts is beyond the scope of this study. In addition, our theoretical analysis is currently restricted to regression problems with fully connected neural networks and does not directly cover transformer architectures, classification settings, or other model families.

\subsection*{AI Use Statement}

The authors used AI to assist with proofreading, grammar and mathematical-notation consistency checks, and the identification of potential derivation, citation, and presentation issues. All suggestions were reviewed and verified by the authors, who take full responsibility for the accuracy, integrity, and content of the manuscript.

\bibliography{reference}

\newpage
\appendix
\setcounter{tocdepth}{2}
\tableofcontents
\allowdisplaybreaks
\newpage

\section{Derivation of Approximate Objective} \label{apx:obj-approx-theory}

\begin{table}[H]
    \centering
    \caption{Summary of the notation used in the approximate-objective derivation.}
    \label{tab:approx-obj-notation}
    \small
    \begin{tabular}{@{}p{0.12\linewidth}p{0.84\linewidth}@{}}
    \toprule
    Symbol & Meaning \\
    \midrule
    $\theta_i$
        & Pre-trained full-precision weight that is \textbf{fixed} and provided as input to the compression procedure. \\
    $\tilde{\theta}_i$
        & Random sparse and quantized weight learned during compression. \\
    $q(\tilde{\theta}_i)$
        & Marginal variational distribution of $\tilde{\theta}_i$, defined by the spike-and-GMM model in Equation~\eqref{eq:marginal_vb}. \\
    $\pi(\tilde{\theta}_i)$
        & Spike-and-slab prior defined in Equation~\eqref{eq:marginal_prior}. \\
    $\gamma_i$; $\lambda$; $\tilde{\lambda}_i$
        & Keep/prune indicator, prior retention probability, and variational retention probability, respectively. \\
    $\phi_k(\theta_i)$
        & Responsibility of GMM component $k$ for weight $i$, evaluated as a function of the fixed weight $\theta_i$. \\
    $\mu_k,\sigma_k^2$
        & Learnable component mean and variance; the means $\mu_k$ define the quantization levels. \\
    $\sigma_0^2$
        & Variance of the Gaussian slab in the spike-and-slab prior. \\
    $\widehat{(\,\cdot\,)}$
        & Quantity estimated after optimization, such as $\hat{\mu}_k$, $\hat{\lambda}_i$, or $\widehat{\phi}_k$. \\
    $K$, $T$, $M$
        & Number of GMM components, total number of weights, and number of posterior samples used for Bayesian averaging, respectively. \\
    \bottomrule
    \end{tabular}
\end{table}

\subsection{An upper bound on the KL divergence between two mixtures}

To simplify the ELBO and validate our approach, we reformulate a key lemma from previous work \citep[Lemma 6.1]{cherief2018consistency}. This Lemma is a tool widely used in signal processing~\cite{DBLP:conf/icassp/HersheyO07}. We provide the
proof for the sake of completeness.
\begin{lemma}[From Lemma 6.1 in~\cite{cherief2018consistency} ] \label{lem:mix}
For any $K>0$, the KL divergence between any two mixture densities $\sum_{k=1}^Kw_kg_k$ and $\sum_{k=1}^K\tilde{w}_k\tilde{g}_k$ is upper bounded by
\begin{equation*}
\KL\left(\sum_{k=1}^Kw_kg_k \Big\| \sum_{k=1}^K\tilde{w}_k\tilde{g}_k\right) \leq  \sum^K_{k=1}w_k\KL(g_k\|\tilde{g}_k)+\sum_{k=1}^K w_k \log\left(\frac{w_k}{\tilde{w}_k}\right) 
\end{equation*}
\begin{proof} 
We expand the KL divergence term by its definition and obtain:
\begin{equation*}
\begin{aligned}
    \KL\left(\sum_{k=1}^Kw_kg_k \Big\| \sum_{k=1}^K\tilde{w}_k\tilde{g}_k\right) &=\int\left(\sum_{k=1}^Kw_kg_k \right)\log \left(\frac{\sum_{k=1}^Kw_kg_k }{\sum_{k=1}^K\tilde{w}_k\tilde{g}_k}\right)\\
    &\le \int \sum_{k=1}^Kw_kg_k\log \left(\frac{w_kg_k }{\tilde{w}_k\tilde{g}_k}\right)  \\
    &= \int \sum_{k=1}^Kw_kg_k\log \left(\frac{g_k }{\tilde{g}_k}\right) +\int \sum_{k=1}^Kw_kg_k\log \left(\frac{w_k }{\tilde{w}_k}\right)\\
    &= \sum_{k=1}^Kw_k\int g_k\log \left(\frac{g_k }{\tilde{g}_k}\right) +\sum_{k=1}^Kw_k\log \left(\frac{w_k }{\tilde{w}_k}\right)\left(\int g_k\right)\\
    &=\sum^K_{k=1}w_k\KL(g_k\|\tilde{g}_k)+\sum_{k=1}^K w_k \log\left(\frac{w_k}{\tilde{w}_k}\right),
\end{aligned}
\end{equation*}
where the first inequality is due to Jensen's inequality and the convexity of the function $x\log(x)$. This completes the proof.
\end{proof}
\end{lemma}

\subsection{Derivation of Approximate Objective}
We aim to approximate the ELBO objective:
\begin{equation}
     \Omega (\tilde{\theta})= -\underbrace{\E_{q(\tilde{\theta})}[\log p(D| \tilde{\theta})] }_{\mathtt{Part\; 1}}+\sum_{i=1}^{T}\KL\left(q(\tilde{\theta}_i) \| \pi(\tilde{\theta}_i)\right). 
\end{equation}
where $\pi(\tilde{\theta}_i)$ is defined in Equation~\eqref{eq:marginal_prior} and $q(\tilde{\theta}_i)$  is defined in Equation~\eqref{eq:marginal_vb}:
\begin{align*}
q(\tilde{\theta}_i) &= \tilde{\lambda}_i \sum_{k=1}^{K} \phi_k(\theta_i) \mathcal{N}(\mu_k, \sigma_k^2) + (1 - \tilde{\lambda}_i)\delta_0,\\
\pi(\tilde{\theta}_i) &= \lambda \mathcal{N}(0, \sigma_0^2) + (1 - \lambda)\delta_0.
\end{align*}
It is important to note that the KL divergence between the variational distribution and the spike-and-slab prior distribution does not have a closed-form solution.

\noindent\textbf{Step 1: Approximate the expected log-likelihood.} The first term $\E_{q(\tilde{\theta})}[\log p(D\mid\tilde{\theta})]$ can be expensive to compute because it requires sampling from the spike-and-GMM distribution. We therefore use a plug-in approximation based on the posterior mean. For each coordinate $i=1,\ldots,T$,
\begin{equation*}
    \E_{q(\tilde{\theta}_i)}[\tilde{\theta}_i]
    =\tilde{\lambda}_i\sum_{k=1}^{K}\mu_k\phi_k(\theta_i).
\end{equation*}
Collecting these coordinate-wise means gives the complete mean parameter vector
\begin{equation*}
    \theta^{\mathrm{avg}} :=\E_{q(\tilde{\theta})}[\tilde{\theta}]
=\left(\tilde{\lambda}_i\sum_{k=1}^{K}\mu_k\phi_k(\theta_i;\pi,\tau)\right)_{i=1}^{T}
    \in\mathbb{R}^{T}.
\end{equation*}
The plug-in approximation is therefore
\begin{equation*}
    \mathtt{Part\;1}\approx\log p(D\mid\theta^{\mathrm{avg}}).
\end{equation*}

\noindent\textbf{Step 2: Upper Bound KL between spike-and-slab distributions.} The KL divergence between the marginal variational posterior and the prior is intractable due to the presence of both the Dirac delta and the mixture components.
To upper-bound the KL divergence between them, we apply Lemma~\ref{lem:mix} by matching component structure:
\begin{align*}
\KL(q \| \pi)& \leq \sum_{j=1}^2 w_j \KL(g_j \| \tilde{g}_j) + \sum_{j=1}^2 w_j \log\left(\frac{w_j}{\tilde{w}_j}\right)
\end{align*}
\begin{align*}
\text{with}\qquad
&w_1 = \tilde{\lambda}_i, 
&&g_1 = \sum_{k=1}^{K} \phi_k(\theta_i)\mathcal{N}(\mu_k,\sigma_k^2),
&&\tilde{w}_1 = \lambda,
&&\tilde{g}_1 = \mathcal{N}(0,\sigma_0^2), \\
&w_2 = 1-\tilde{\lambda}_i,
&&g_2 = \delta_0,
&&\tilde{w}_2 = 1-\lambda,
&&\tilde{g}_2 = \delta_0 .
\end{align*}
Substituting into the bound, we obtain:
\begin{align*}
\KL(q(\tilde{\theta}_i) \| \pi(\tilde{\theta}_i)) 
\leq &\tilde{\lambda}_i \KL\left(\sum_{k=1}^{K} \phi_k(\theta_i) \mathcal{N}(\mu_k, \sigma_k^2)  \Big\| \mathcal{N}(0, \sigma_0^2)\right) + (1 - \tilde{\lambda}_i) \underbrace{\KL(\delta_0 \| \delta_0)}_{= 0}\\
&+ \tilde{\lambda}_i \log\frac{\tilde{\lambda}_i}{\lambda} + (1 - \tilde{\lambda}_i)\log\frac{1 - \tilde{\lambda}_i}{1 - \lambda}.
\end{align*}
Combining the terms, we have:
\begin{align*}
\KL(q(\tilde{\theta}_i) \| \pi(\tilde{\theta}_i)) 
\leq \tilde{\lambda}_i  \KL\left(\sum_{k=1}^{K}\phi_k(\theta_i)\mathcal{N}(\mu_k, \sigma^2_k)  \Big\| \mathcal{N}(0, \sigma_0^2)\right)+\KL(\mathrm{Bern}(\tilde{\lambda}_i) \| \mathrm{Bern}(\lambda)),
\end{align*}
Note that the first term on the right-hand side, which is the KL divergence between the GMM and the Gaussian distribution, does not have a closed form. But it can be further upper-bounded as:
\begin{align*}
    \KL&\left(\sum_{k=1}^{K}\phi_k(\theta_i)\mathcal{N}(\mu_k, \sigma^2_k)\Big\| \mathcal{N}(0, \sigma^2_0)\right) = \KL\left(\sum_{k=1}^{K}\phi_k(\theta_i)\mathcal{N}(\mu_k, \sigma^2_k)\Big\|\sum_{k=1}^K \phi_k(\theta_i)\mathcal{N}(0, \sigma^2_0)\right) \\
    \leq& \sum_{k=1}^{K}\phi_k(\theta_i)\KL\left(\mathcal{N}(\mu_k,\sigma^2_k) \| \mathcal{N}(0, \sigma^2_0)\right) + \sum_{k=1}^{K} \phi_k(\theta_i)\underbrace{\log\left(\frac{\phi_k(\theta_i)}{\phi_k(\theta_i)}\right)}_{=0} \\
    =& \underbrace{\sum_{k=1}^{K}\phi_k(\theta_i)\KL\left(\mathcal{N}(\mu_k,\sigma^2_k) \| \mathcal{N}(0, \sigma^2_0)\right)}_{\mathtt{Part~2}},
\end{align*}
where the inequality is obtained by Lemma \ref{lem:mix}. 
Empirically, we approximate the mixture KL by evaluating only the dominant component:
\begin{equation*}
\mathtt{Part\; 2}\approx \KL(\mathcal{N}(\mu_{k_i^*},\sigma^2_{k_i^*}) \| \mathcal{N}(0, \sigma^2_0)),\qquad\text{where } k_i^* = \argmax_{1\le k \le K} \phi_k(\theta_i)
\end{equation*}
We approximate the inner sum over $k$ using the maximum-weight component, which is the $k^*$-th component.

A small temperature is needed to avoid a flat posterior distribution, which could introduce large differences between the training phase and inference phase.

Finally, putting all approximations together, we obtain:
\begin{align*}
    \Omega_{\mathtt{apx}}(\tilde{\theta}) = -\log p(D\mid\theta^{\mathrm{avg}}) +\sum_{i=1}^T \KL(\mathrm{Bern}(\tilde{\lambda}_i)  \|  \mathrm{Bern}(\lambda))
+\sum_{i=1}^{T}\tilde{\lambda}_i\KL(\mathcal{N}(\mu_{k_i^*},\sigma^2_{k_i^*}) \| \mathcal{N}(0, \sigma^2_0)),
\end{align*}
where $\theta^{\mathrm{avg}}$ is the mean parameter vector defined in Step~1, and $k_i^* = \argmax_{1\le k\le K} \phi_k(\theta_i)$. We thus obtain the result shown in Equation~\eqref{eq:approx_elbo}.

\newpage

\newpage
\section{Proof of Theorem \ref{thm:convergence}} \label{apx:converge-obj}

Consider an $ L$- hidden-layer fully connected neural network with the ReLU activation function $\sigma_b: \mathbb{R}^d\rightarrow \mathbb{R}^d$ defined as $\sigma_b(X) = \max\{0, X-b\}$ on some dimension $d$ and parameter $b$. The number of neurons in each layer is defined as $p_i$ for $i=1, \dots, L$. The weights and biases are denoted by $W_i \in \mathbb{R}^{N \times N}$ and $b_i \in \mathbb{R}^{N}$. Thus, given the parameters $\mathbf{p} = (p_1, \cdots, p_L)$, let $\theta$ denote the vector obtained by stacking all entries of the weight matrices $W_i$ and bias vectors $b_i$, then the fully connected network can be presented as:
\begin{equation*}
    f_{\theta}(X)=W_{L+1}\sigma_{b_{L}}(W_{L}\sigma_{b_{L-1}}\ldots \sigma_{b_1}(W_1X)) + b_{L+1}.
\end{equation*} 
The DNN $f_{\theta}$ also introduces a probability measure of the data, which we denote as $P_{\theta}$, and $p_{\theta}$ is the corresponding density function; $p_{\theta}(D)$ would be the likelihood of the data $D$. 

One can define the sparse parameter space with sparsity parameter $s$ as $\{\theta\in \R^T: \norm{\theta}_0 < s\}$, where $\theta$ has only $s$ many non-zero entries. 
Then we can further introduce the sparse and quantized weights space $H(T, s, K)$ as follows:
\begin{align*}
    \mathcal{I}(T, s, K) &= \{I=[r_1,\dots, r_T]^\top | \norm{I}_0 \leq s, r_i\in\{0,1\}^K, \norm{r_i}_0\leq1\}, \\
    H(T,s,K) &=\left\{\theta \in \R^{T} | \theta=I\cdot E, E\in [-B, B]^{K},I\in\mathcal{I}(T, s, K)\right\},
\end{align*}
where $B$ is some constant that satisfies $B > 2$ and $\mathcal{I}(T, s, K)$ is the indexing space; each element $I\in\mathcal{I}(T, s, K)$ consists of $s$ many $ K$-dimensional one-hot rows, and the remaining $T-s$ rows are zero vectors indicating the corresponding weight is pruned. In such a way, any $\theta \in H(T, s, K)$ satisfies that $\norm{\theta}_0 \leq s$ and $\theta$ only have $K$ many distinct entry values then the DNN $f_{\theta}(\cdot)=f(\cdot; \theta)$ is sparse and quantized.  %In Theorem \ref{thm:convergence}, we show that our variational learning procedure with high probability could find a closed enough, sparse, and quantized DNN from the sparse and quantized weight space $H(T, s, K)$to the true function $f_0(\cdot)$. 
The following conditions are assumed, similarly to \cite{bai2020efficient}:
\begin{condition} \label{cond:constant_width}
$p_i\equiv N\in\mathbb Z^{+}$ that can depend on n, and $\lim T=\infty$. 
\end{condition}
\begin{condition}\label{cond:activation}
$\sigma(x)$ is 1-Lipschitz continuous. 
\end{condition}
\begin{condition}\label{cond:sigma}
The hyperparameter $\sigma_0^2$ is set to be some constant, and $\lambda$ satisfies
\begin{align*}
\log\left(\frac{1}{\lambda}\right) &= O\left((L+1)\log N + \log(p\sqrt{n/s^*})\right)\\
\log\left(\frac{1}{1-\lambda}\right) &= O\left(\frac{s^*}{T}\left((L+1)\log N + \log(p\sqrt{n/s^*})\right)\right)
\end{align*}
\end{condition}

% \begin{condition}\label{cond:bounded}
    
% \end{condition}

\begin{condition}\label{cond:order}
     $\max\{s^*\log(p\sqrt{n/s^*}), (L+1)s^*\log N\} = o(n)$ and $r_n^*\asymp \xi_n^*$.
\end{condition}
The ``oracle'' sparsity $s^*$ is defined in Equation~\eqref{eq:def_sstar}.

\begin{definition}
The true function $f_0$ defined in Equation~\eqref{eq:regression} is $\beta$-H\"older smoothness if
\begin{align}
    \nonumber f_0 \in \mathcal{C}^{\beta}_{ p}(M) = \Big\{&f:[0,1]^{p}\rightarrow \R : \sum_{\alpha: |\alpha|<\beta}\|\partial^{\alpha} f\|_{\infty}+ \sum_{\alpha: |\alpha|=\lfloor\beta\rfloor }\sup_{\stackrel{x, y \in [0, 1]^{p}}{x \neq y}}
    \frac{|\partial^{\alpha} f(x) - \partial^{\alpha} f(y)|}{\|x-y\|_\infty^{\beta-\lfloor \beta \rfloor}} \leq M
    \Big\}.
\end{align}
for some constant $M>0$.
\end{definition}

Following previous paper~\citep{bai2020efficient}, we define:
\begin{align}\label{eq:def_sstar}
s^*=\arg\min_s\{r_n(L,p,s)+\xi_n(L,p,s)\}
\end{align}
where
\begin{align*}
r_n(L,p,s)&= ((L+1)s/n)\log N + (s/n)\log \left(p\sqrt{n/s}\right)\\
\xi_n(L,p,s) &= \inf_{\theta\in H(T,s, K), \|\theta\|_{\infty}\le B }\|f_\theta-f_0\|^2_\infty.
\end{align*}
Correspondingly, we define $
r_n^*= r_n(L,p,s^*),
 \xi_n^* = \xi_n(L,p,s^*)$.

In this section, we reformulate the variational distribution by introducing a latent index variable. For any $q(\tilde{\theta}) \in \mathcal{F}$, it has the following equivalent form:
\begin{equation}\label{eq:nmodel}
 \begin{aligned}
    \tilde{\theta}_i|z_i, \gamma_i &\sim \gamma_i\sum_{k=1}^{K}\mathbbm{1}\{z_i=k\}\mathcal{N}(\mu_k, \sigma_k^2)+(1-\gamma_i)\delta_0, \\
    z_i &\sim \text{Categorical}(\phi_1(\theta_i), \phi_2(\theta_i), \dots, \phi_K(\theta_i)), \\
    \gamma_i &\sim \text{Bernoulli}(\lambda_i),
\end{aligned}   
\end{equation}
In addition, for theoretical convenience, we further restrict the variational family to satisfy 
\begin{condition}\label{cond:bound} 
$|\mu_k| \leq B$ and $\sigma_k^2 \leq \frac{1}{2\log(T/r_n^*\log^2(n))}$.
\end{condition}
Note that the requirement of $|\mu_k| \leq B$ is fairly reasonable, as most of the existing approximation results \citep{cherief2020convergence} only need bounded DNN weights. 

We restate a formal version of our Theorem \ref{thm:convergence} as follows:

\begin{theorem}\label{thm:formal_convergence}
Under Conditions \ref{cond:constant_width}-\ref{cond:activation} and \ref{cond:order}-\ref{cond:bound}, 
Let $\sigma_0^2$ be a constant and $-\log\lambda =\log(T)+\delta[(L+1)\log N+\log\sqrt np]$ for any constant $\delta>0$, 
Then with high probability:
    \begin{equation}
        \int_{\R^{T}} d^2(P_{\theta}, P_0) \widehat{q}(\theta)\mathit{d}\theta \leq C \varepsilon_n^{*2} + C'(r_n^*+\xi_n^*),
    \end{equation}
where $d(\cdot,\cdot)$ denotes the Hellinger distance, and $C$ and $C'$ are some constants. 
\end{theorem}

\begin{proof}
The convergence in squared Hellinger distance follows directly from Lemmas \ref{lem:bound_elbo} and Lemma~\ref{lem:lmbound1}, as the chosen value of $\lambda$ meets the necessary assumptions. 
\end{proof}

\noindent\textbf{Remark}. Compared to prior results, Lemma \ref{lem:bound_elbo} demonstrates that a spike-and-slab prior combined with a Gaussian mixture model (GMM) with finitely many components can effectively approximate the true underlying function. In contrast, Lemma \ref{lem:lmbound1} establishes that the statistical estimation error of the spike-and-GMM variational distribution vanishes as the sample size $n \to \infty$.

\newpage

\subsection{Proof of Lemma \ref{lem:bound_elbo}}  \label{sec:proof-bound-elbo}

\begin{proof}
Let $\theta^*=\arg\min_{\theta\in H(L, N, s, K)}\|f_\theta-f_0\|^2_{\infty}$. By definition, they are only $K$ many unique non-zero number in $\theta^*$, denoted as $\mu_k^* \in \R$, for $k = 1,\dots, K$. In other words, for any $\theta_i^* \neq 0$,  $\theta_i^*$ must choose from the quantization set $\mathcal{Q} = \{\mu_1^*, \dots, \mu_K^*\}$, and we denote the choice index of $\theta_i^*$ from $\mathcal{Q}$ as $(i)$, (i.e., $\theta_i^*=\mu_{(i)}^*$). Now, given $\theta^*$, we construct $q^*(\tilde{\theta})$ as follow:
\begin{align*}
    \tilde{\theta}_i|z_i^*, \gamma_i^*&\sim\gamma^*_i \sum_{k=1}^K\mathbbm{1}\{z_i^*=k\}\mathcal{N}(\mu_k^*, \sigma_n^2)+(1-\gamma_i^*)\delta_0, \\
    z_i^* &\sim \text{Categorical}(\phi_1(\theta_i^*), \phi_2(\theta_i^*), \dots, \phi_K(\theta_i^*)),\\
\gamma_i^*&\sim\text{Bernoulli}(\psi^*_i), \qquad \psi^*_i=\mathbbm{1}\{\theta_i^*\neq 0\}
\end{align*}
where $\sigma_n^2= \frac{s^*}{32n}\log(3N)^{-1}(2BN)^{-2L}\{(p+1+\frac{1}{BN-1})+\frac{1}{(2BN)^2-1}+\frac{2}{(2BN-1)^2}\}^{-1}$.
\footnote{Notice that $\sigma_n^2$ satisfies Condition \ref{cond:bound} as with sufficiently large $n, N$ and $L$, $1/n \leq 1/\log(n)$ and $1/N^L < 1/\log(T) = 1/\log(NL)$.}
Thus, we can have the following marginal distribution:
\begin{equation*}
    q^*(\theta) = \sum\mathbf{1}\{\gamma=\gamma^*\}\prod_{i=1}^T\gamma_i \left[ \phi_{(i)}(\theta_i^*)\mathcal{N}(\mu_{(i)}^*,\sigma_n^2)+\sum_{k\neq (i)}\phi_{k}(\theta_i^*)\mathcal{N}(\mu_{k}^*,\sigma_n^2)\right]+(1-\gamma_i^*)\delta_0
\end{equation*}
Next we first need to bound $\int \norm{f_\theta-f_{\theta^*}}_{\infty}^2 q^*(\mathit{}\theta)$. We can also write $\theta^* = \{W_1^*, b_1^*, \dots, W_{L+1}^*, b_{L+1}^*\}$, then we define the following terms as:
\begin{align*}
    \tilde{W}_l &= \sup_{i, j} |W_{l, i,j}-W_{l,i, j}^*|,\\
    \tilde{b}_l &= \sup_{i} |b_{l,i}-b_{l,i}^*|.
\end{align*}
Then, following the proof in~\cite[Proof of Theorem 7]{cherief2020convergence}, we can have the following:
\begin{align}
    \nonumber&\int \norm{f_\theta-f_{\theta^*}}_{\infty}^2 q^*(\mathit{d}\theta)  \\
    \nonumber&\leq 2 N^{2L-2}\bigg(d+1+\frac{1}{BN-1}\bigg)^2 \bigg( \sum_{l=1}^L B^{2l-2} \prod_{v=l+1}^L \int (B+\tilde{W}_v)^2 q(\mathit{d} \theta) \int \tilde{W}_l^2 q_l(\mathit{d} \theta_l) \\
    \nonumber&  + 2 \sum_{l=1}^L \sum_{k=1}^{l-1} B^{l-1} B^{k-1} \prod_{v=l+1}^L \int (B+\tilde{W}_v)^2 q(\mathit{d} \theta) \int \tilde{W}_l q_l(\mathit{d} \theta_l) \prod_{v=k+1}^{l} \int (B+\tilde{W}_v) q(\mathit{d} \theta) \int \tilde{W}_k q(\mathit{d} \theta) \bigg) \\
    \nonumber&  + 2 \bigg( \sum_{l=1}^L D^{2(L-l)} \prod_{v=l+1}^L \int (B+\tilde{W}_v)^2 q(\mathit{d} \theta) \int \tilde{b}_l^2 q(\mathit{d} \theta) \\
    &  + 2 \sum_{l=1}^L \sum_{k=1}^{l-1} D^{L-l} D^{L-k} \prod_{v=l+1}^L \int (B+\tilde{W}_v)^2 q(\mathit{d} \theta) \int \tilde{b}_l q(\mathit{d} \theta) \prod_{v=k+1}^{l} \int (B+\tilde{W}_v) q(\mathit{d} \theta) \int \tilde{b}_k q(\mathit{d} \theta) \bigg).\label{eq:20}
\end{align}
Then next we need to upper bound the term:
\begin{equation*}
    \int \tilde{W}_l q^*(\mathit{d}\theta) = \int \sup_{i,j} |W_{l, i, j}-W_{l,i,j}^*| q^*(\mathit{d}\theta)
\end{equation*}
We first bound the following, for some $t > 0$,
\begin{align}
    \nonumber \exp\Big(\E[t\sup_{i, j}|W_{l, i,j}-W_{l,i,j}^*|]\Big) 
    &\leq \E\sup_{i,j}\exp\Big(t|W_{l, i,j}-W_{l,i,j}^*|\Big)\\ 
    \nonumber & \leq N^2\E\Big[\exp\Big(t|W_{l, i,j}-W_{l,i,j}^*|\Big)\Big] \\ 
    \nonumber & = N^2\E\Bigg[ \E_{\mathcal{N}(\mu_k^*, \sigma_n^2)}\Big[\exp(t|W_{l, i,j}-W_{l,i,j}^*|) | z_{l,i,j}  = k\Big] \Bigg] \\ 
    \nonumber & = N^2 \sum_{k=1}^{K}\phi_k(W_{l,i,j}^*)\E_{\mathcal{N}(\mu_k^*,\sigma_n^2)}\Big[\exp\Big(t|W_{l, i, j}-W_{l, i, j}^*|\Big)\Big] \\ 
    \nonumber &= N^2 \Bigg( \phi_{(i)}(W_{l,i,j}^*)\E_{\mathcal{N}(\mu_{(i)}^*, \sigma_n^2)}\Big[\exp\Big(t|W_{l,i,j}-{W}_{l,i,j}^*|\Big)\Big] \\
    &\qquad \qquad +\sum_{k\neq(i)}^{K}\phi_{k}(W_{l,i,j}^*)\E\Big[\exp\Big(t|W_{l,i,j}-W_{l,i,j}^*|\Big)\Big]\Bigg) \label{eq:21}
\end{align}
Notice that by definition $W_{l,i,j}^* = \mu_{(i)}^*$, thus we can bound the first term as:
\begin{align*}
    \E_{\mathcal{N}(\mu_{(i)}^*, \sigma_n^2)}\Big[\exp\Big(t|W_{l,i,j}-{W}_{l,i,j}^*|\Big)\Big] &= \E_{\mathcal{N}(\mu_{(i)}^*, \sigma_n^2)}\Big[\exp\Big(t|W_{l,i,j}-\mu_{(i)}^*|\Big)\Big] \\
    &=\int_0^\infty P(\exp(t|W_{l,i,j}-\mu_{(i)}^*|>x))\mathit{d}x \\
    &=\int_0^\infty P\left(|W_{l,i,j}-\mu_i^*|>\frac{\log x}{t}\right)\mathit{d}x\\
    &=\int_0^\infty 2 P\left(W_{l,i,j}-\mu_{(i)}^*>\frac{\log x}{t}\right)\mathit{d}x\\
    &=\int_0^\infty 2 P\left(z>\frac{\log x}{t \sigma_n}\right)\mathit{d}x\\
    &=2\exp\left(\frac{t^2\sigma^2_n}{2}\right)
\end{align*}
Next, we bound the second term of the Equation~\eqref{eq:21}:
\begin{align*}
    \mathbb{E}_{\mathcal{N}(\mu_{k}^*, \sigma_n^2)}\Big[\exp(t|W_{l,i,j}- \mu_{(i)}^*|)\Big] &=\int_{0}^{\infty} P\Big(\exp(t|W_{l,i,j}-\mu_{(i)}^*|> x)\Big)\mathit{d}x \\
    &=\int_{0}^{\infty} 2P\Big(W_{l,i,j}-\mu_{(i)}^*> \frac{\log x}{t}\Big)\mathit{d}x \\
    &=\int_{0}^{\infty}2P\Big(W_{l,i,j}-\mu_k^*+\mu_k^*-\mu_{(i)}^* > \frac{\log x}{t}\Big)\mathit{d}x\\
    &=\int_{0}^{\infty}2P\Big(\sigma_n z+\mu_k^*-\mu_{(i)}^* > \frac{\log x}{t}\Big)\mathit{d}x\\
    &=2\exp\left(t(\mu_{k}^*-\mu_{(i)}^*)+\frac{\sigma_n^2t^2}{2}\right) \\
    &\leq 2\exp\Big(4t+\frac{\sigma_n^2t^2}{2}\Big)
\end{align*}
Notice that the last inequality is because of $\sup_{i,j}(\mu_i^*-\mu_j^*) \leq 4$.
And by choosing $t=\sqrt{2\log(3N^2)}/\sigma_n$, the Equation~\eqref{eq:21}, can be bounded by:
\begin{align*}
    \exp\left(\E\left[t\sup_{i, j}|W_{l, i,j}-W_{l,i,j}^*|\right]\right) &\leq 2N^2\Bigg(\phi_{(i)}(W_{l,i,j}^*)\exp(t^2\sigma_n^2/2)+\sum_{k\neq(i)}\phi_k(W_{l,i,j}^*)\exp(4t+t^2\sigma_n^2/2)\Bigg) \\
    &=2N^2\Big(\exp(t^2\sigma_n^2/2)+\sum_{k\neq(i)}\phi_k(W_{l,i,j}^*)\exp(4t+t^2\sigma_n^2/2)\Big)\\
    & \leq 2N^2\Big(\exp(t^2\sigma_n^2/2)+1/2\exp(t^2\sigma_n^2/2)\Big) \\
    & = 3N^2\exp\Big(t^2\sigma_n^2/2\Big) 
\end{align*}
The second inequality is obtained by setting $\tau$ such that $\sum_{k\neq(i)}\phi_k(W_{l,i,j}^*) \leq 1/2$. Note that given a fixed DNN structure, such a $\tau$ always exists as the $\sum_{k\neq(i)}\phi_k(W_{l,i,j}^*)$ decrease monotonically to zero as $\tau$ decrease.
Thus, we can have 
\begin{align*}   
    \E\Big[\sup_{i,j}|W_{l,i,j}-W_{l,i,j}^*|\Big] &\leq \frac{3N^2}{t} + \frac{t \sigma_n^2}{2} =\sqrt{2\sigma_n^2\log(3N^2)} \leq \sqrt{8\sigma_n^2\log(3N)}
\end{align*}

Next we bound $\int \tilde{W}_l^2 q^*(\mathit{d}x)$, following similar procedure, for some $t > 0$, we can have:
\begin{align}
    \nonumber\exp\Big(\E\big[t\sup_{i,j}(W_{l,i,j}-W_{l,i,j}^*)^2\big]\Big) &\leq N^2\E\Big[\exp(t(W_{l,i,j}-W^*_{l,i,j})^2)\Big]\\
    \nonumber&\leq N^2 \E\Bigg[\E_{\mathcal{N}(\mu_k, \sigma_n^2)}[t(W_{l,i,j}-W_{l,i,j}^*)^2|z_{l,i,j}=k]\Bigg] \\
    \nonumber &=N^2\Big(\phi_{(i)}(W_{l,i,j}^*)\E_{\mathcal{N}({\mu_{(i)}^*, \sigma_n^2})}\Big[\exp(t(W_{l,i,j}-W_{l,i,j}^*)^2)\Big] \\
    &\qquad \qquad+\sum_{k\neq(i)}\E_{\mathcal{N}(\mu_k, \sigma_n^2)}\Big[\exp(t(W_{l,i,j}-W_{l,i,j}^*)^2)\Big]\Big)\label{eq:29}
\end{align}
Note that with a slight abuse of notation, the $z_{l, i,j}$ in the above equation means the latent variable $z_i$ (introduced in (\ref{eq:nmodel})) corresponding to the weight $W_{l, i,j}$.
The first term of Equation~\eqref{eq:29} can be bounded for $t<\frac{1}{2\sigma_n^2}$ as follow:
\begin{align*}
    \E_{\mathcal{N}(\mu_{(i)}^*,\sigma_n^2)}\Big[\exp\Big(t(W_{l,i,j}-W_{l,i,j}^*)^2\Big)\Big] &= \E_{\mathcal{N}(\mu_{(i)}^*,\sigma_n^2)}\Big[\exp\Big(t(W_{l,i,j}-\mu_{(i)}^*)^2\Big)\Big]=\frac{1}{\sqrt{1-2t\sigma_n^2}}.
\end{align*}
And the second term of Equation~\eqref{eq:29} can be bounded as:
\begin{align*}
    \E_{\mathcal{N}(\mu_k, \sigma_n^2)}\Big[\exp\Big(t(W_{l,i,j}-W_{l,i,j}^*)^2\Big)\Big] &=  \E_{\mathcal{N}(\mu_k, \sigma_n^2)}\Big[\exp\Big(t(W_{l,i,j}-\mu_{(i)}^*)^2\Big)\Big] \\
    &= \E_{\mathcal{N}(\mu_k, \sigma_n^2)}\Big[ \exp\Big(t(W_{l,i,j}-\mu_{k}^*+\mu_{k}^*-\mu_{(i)}^*)\Big)\Big] \\
    &= \E\Big[\exp\Big(t(\sigma_n z+\mu_{k}^*-\mu_{(i)}^*)^2\Big)\Big] \\
    &= \frac{1}{\sqrt{1-2t\sigma_n^2}}\exp\Big(\frac{(\mu_{k}^*-\mu_{(i)}^*)^2 t}{1-2t\sigma_n^2}\Big) \\
    &\leq \frac{1}{\sqrt{1-2t\sigma_n^2}}\exp\Big(\frac{16t}{1-2t\sigma_n^2}\Big)
\end{align*}
The last inequality is again because of the property that $\sup_{i,j}|\mu_{i}^*-\mu_j^*| \leq 4$.
Thus Equation~\eqref{eq:29} can be bounded by:
\begin{align*}
    \exp\Big(\E\big[t\sup_{i,j}(W_{l,i,j}-W_{l,i,j}^*)^2\big]\Big)&\leq N^2\Bigg(\frac{1}{\sqrt{1-2t\sigma_n^2}}+\sum_{k \neq (i)}\phi_{k}(W_{l,i,j}^*)\frac{1}{\sqrt{1-2t\sigma_n^2}}\exp\Big(\frac{16t}{1-2t\sigma_n^2}\Big)\Bigg) \\
    &\leq 2N^2 \frac{1}{\sqrt{1-2t\sigma_n^2}},
\end{align*}
where the last inequality is because of choosing a small $\tau$ such $\sum_{k\neq(i)}\phi_{k}(W_{l,i,j}^*) \leq \exp(\frac{16t}{1-2t\sigma_n^2})$, such $\tau$ always exist since $\sum_{k\neq(i)}\phi_{k}(W_{l,i,j}^*)$ decrease monotonically to $0$ as $\tau \to 0$.

Thus we can bound $\int \tilde{W}_l^2 q^*(\mathit{d}x)$ by the following:
\begin{align*}
    \int \tilde{W}_l^2 q^*(\mathit{d}x) = \int \sup_{i,j}(W_{l,i,j}-W_{l,i,j}^*)^2\mathit{d}x 
    &\leq \log(N^2)/t+\log\Big(\frac{2}{\sqrt{1-2t\sigma_n^2}}\Big)/t \\
    &= 4\log(N^2)\sigma_n^2+4\log\left(\frac{2}{\sqrt{\frac{1}{2}}}\right)\sigma_n^2 \\
    & \leq 8\sigma_n^2 \log(3N)
\end{align*}
By choosing $\sqrt{8\sigma_n^2\log(3N)} \leq B$, we can have:
\begin{align*}
    &\int \Big(B+\tilde{W}_l\Big) q^*(\mathit{d}\theta) \leq 2B ,\\
    &\int\Big(B+\tilde{W}_l\Big)^2 q^*(\mathit{d}\theta) \leq B^2+2B\sqrt{8\sigma_n^2\log(3N)}+8\sigma_n^2\log(3N) \leq 4B^2. 
\end{align*}
Similarly, we can have the following:
\begin{align*}
    \int \tilde{b}_l q^*(\mathit{d}\theta) &\leq \sqrt{8\sigma_n^2\log(3N)}, \\
    \int \tilde{b}_l^2 q^*(\mathit{d}\theta) &\leq 8\sigma_n^2\log(3N).
\end{align*}
Combined with Equation~\eqref{eq:20}, we can have:
\begin{align*}
    \int \norm{f_\theta-f_{\theta^*}}_{\infty}^* q^*(\mathit{d}\theta) &\leq 2 N^{2L-2}\bigg(p+1+\frac{1}{BN-1}\bigg)^2 \bigg( \sum_{\ell=1}^L B^{2\ell-2} (4B^2)^{L-\ell} 8\sigma_n^2\log (3N) \\
    & \qquad + 2 \sum_{\ell=1}^L \sum_{k=1}^{\ell-1} B^{\ell-1} B^{k-1} (4B^2)^{L-\ell} \sqrt{8\sigma_n^2\log (3N)} (2B)^{\ell-k} \sqrt{8\sigma_n^2\log (3N)} \bigg) \\
    & \qquad + 2 \bigg( \sum_{\ell=1}^L N^{2(L-\ell)} (4B^2)^{L-\ell} 8\sigma_n^2\log (3N) \\
    & \qquad + 2 \sum_{\ell=1}^L \sum_{k=1}^{\ell-1} N^{L-\ell} N^{L-k} (4B^2)^{L-\ell} \sqrt{8\sigma_n^2\log (3N)} (2B)^{\ell-k} \sqrt{8\sigma_n^2\log (3N)} \bigg).
\end{align*}
With some algebra, we can have:
\begin{align*}
    \int & \|f_{\theta}-f_{\theta^*}\|_2^2 q_n^*(\mathit{d} \theta) \\
    & \leq 2 N^{2L-2}\bigg(d+1+\frac{1}{BN-1}\bigg)^2 \bigg( B^{2L-2} 8\sigma_n^2\log (3N) \sum_{\ell=0}^{L-1} 4^\ell   + 2 B^{2L-2} 8\sigma_n^2\log (3D) \sum_{\ell=1}^L \sum_{k=1}^{\ell-1} 2^{L-\ell} 2^{L-k} \bigg) \\
    & \quad + 2 \bigg( 8\sigma_n^2\log (3N) \sum_{\ell=1}^L (2BN)^{2L-2\ell} + 16\sigma_n^2\log (3N) \sum_{\ell=1}^L \sum_{k=1}^{\ell-1} (2BN)^{L-\ell} (2BN)^{L-k} \bigg) \\
    & \leq 2 N^{2L-2}\bigg(p+1+\frac{1}{BN-1}\bigg)^2 \bigg( B^{2L-2} 8\sigma_n^2\log (3N) \frac{4^L-1}{4-1} + 2 B^{2L-2} 8\sigma_n^2\log (3N) \sum_{\ell=1}^L 2^{L-\ell} 2^{L-\ell+1} \sum_{k=0}^{\ell-2} 2^k \bigg) \\
    & \quad + 2 \bigg( 8\sigma_n^2\log (3N) \sum_{\ell=0}^{L-1} (2BN)^{2\ell} + 16\sigma_n^2\log (3N) \sum_{\ell=1}^L (2BD)^{L-\ell} (2BN)^{L-\ell+1} \sum_{k=0}^{\ell-2} (2BN)^k \bigg) \\
    & \leq 2 N^{2L-2}\bigg(d+1+\frac{1}{BN-1}\bigg)^2 \bigg( B^{2L-2} 8\sigma_n^2\log (3N) \frac{4^L}{3} 
    + 2 B^{2L-2} 8\sigma_n^2\log (3N) \sum_{\ell=1}^L 2^{L-\ell} 2^{L-\ell+1} 2^{\ell-1} \bigg) \\
    & \quad + 2 \bigg( 8\sigma_n^2\log (3N) \frac{(2BN)^{2L}}{(2BN)^2-1} + 16\sigma_n^2\log (3N) \sum_{\ell=1}^L (2BN)^{L-\ell} (2BN)^{L-\ell+1} \frac{(2BN)^{\ell-1}}{2BN-1} \bigg) \\
    & \leq 2 N^{2L-2}\bigg(d+1+\frac{1}{BN-1}\bigg)^2 \bigg( B^{2L-2} 8\sigma_n^2\log (3N) \frac{4^L}{3} + 2 B^{2L-2} 8\sigma_n^2\log (3N) 2^L \sum_{\ell=0}^{L-1} 2^\ell \bigg) \\
    & \quad + 2 \bigg( 8\sigma_n^2\log (3N) \frac{(2BN)^{2L}}{(2BN)^2-1} + 16\sigma_n^2\log (3N) \sum_{\ell=0}^{L-1} (2BN)^\ell \frac{(2BN)^L}{2BN-1} \bigg) \\
    & \leq 2 N^{2L-2}\bigg(p+1+\frac{1}{BN-1}\bigg)^2 \bigg( B^{2L-2} 8\sigma_n^2\log (3N) \frac{4^L}{3} + 2 B^{2L-2} 8\sigma_n^2\log (3N) 2^{2L} \bigg) \\
    & \quad + 2 \bigg( 8\sigma_n^2\log (3N) \frac{(2BN)^{2L}}{(2BN)^2-1} + 16\sigma_n^2\log (3N) \frac{(2BN)^{2L}}{(2BN-1)^2} \bigg) \\
    & = 2 N^{2L-2}\bigg(p+1+\frac{1}{BN-1} \bigg)^2 8\sigma_n^2\log(3N) \bigg( B^{2L-2} \frac{4^L}{3} + 2 B^{2L-2} 2^{2L} \bigg) \\
    & \quad + 2 \bigg( \frac{(2BN)^{2L}}{(2BN)^2-1} + 2 \frac{(2BN)^{2L}}{(2BN-1)^2} \bigg) 8\sigma_n^2\log(3N),
\end{align*}
as $BN > 2$,
\begin{equation*}
\begin{aligned}
    \int & \|f_{\theta}-f_{\theta^*}\|_{\infty}^2 q^*(\mathit{d} \theta) \\
    & \leq 16 \sigma_n^2\log(3N) \bigg\{ N^{2L-2} \bigg(p+1+\frac{1}{BN-1} \bigg)^2 \frac{7}{3} B^{2L-2} 2^{2L} + (2BN)^{2L} \bigg( \frac{1}{(2BN)^2-1} + \frac{2}{(2BN-1)^2} \bigg) \bigg\} \\
    & = 16 \sigma_n^2\log(3N) \bigg\{ (2BN)^{2L} \frac{1}{(BN)^2} \bigg(p+1+\frac{1}{BN-1} \bigg)^2 \frac{7}{3}  + (2BN)^{2L} \bigg( \frac{1}{(2BN)^2-1} + \frac{2}{(2BN-1)^2} \bigg) \bigg\} \\
    & \leq 16 \sigma_n^2\log(3N) (2BN)^{2L} \bigg\{ \bigg(p+1+\frac{1}{BN-1} \bigg)^2 + \frac{1}{(2BN)^2-1} + \frac{2}{(2BN-1)^2} \bigg\} \\
    & = \frac{s^*}{2n} \leq r_n^*
\end{aligned}
\end{equation*}
The last equality is due to the definition of $\sigma_n$, and the last inequality is due to the definition of $r_n^*$.

In the next step, we aim to bound the integral $\int_{\R^T} l_n(P_0, P_\theta)\mathit{d}\theta$. Note that by definition:
\begin{align*}
    l_n(P_0, P_\theta) &= \frac{1}{2\sigma_{\varepsilon}^2}(\norm{Y - f_{\theta}(X)}^2_2 - \norm{Y - f_{0}(X)}^2_2) \\
    &= \frac{1}{2\sigma^2_{\varepsilon}} (\|Y - f_{0}(X) + f_0(X)-f_{\theta}(X))\|^2_2 - \|Y - f_{0}(X)\|^2_2)\\
	&= \frac{1}{2\sigma^2_{\varepsilon}}(\|f_{\theta}(X)-f_0(X)\|^2_2 + 2\langle Y-f_0(X), f_0(X) - f_\theta(X)\rangle),
\end{align*}
We can define the following:
\begin{align*}
\mathcal{R}_1 &= \int_{\R^T} \|f_{\theta}(X)-f_0(X)\|^2_2 q^*(\theta)(\mathit{d}\theta),\\
\mathcal{R}_2 &= \int_{\R^T}\langle Y-f_0(X), f_0(X)-f_{\theta}(X)\rangle q^*(\theta)(\mathit{d}\theta).
\end{align*}
Since $\|f_{\theta}(X)-f_0(X)\|^2_2 \leq n \|f_{\theta}-f_0\|^2_{\infty}\leq n(r_n^*+\|f_{\theta^*}-f_0\|^2_{\infty})$, it follows that
\begin{align*}
\mathcal{R}_1 \leq n r_n^* + n\|f_{\theta^*}-f_0\|^2_{\infty}.
\end{align*}
Given $Y - f_0(X)=\varepsilon\sim\mathcal{N}(0,\sigma_{\varepsilon}^2I)$, we have
\begin{align*}
\mathcal{R}_2=\varepsilon^T\int_{\Theta}(f_0(X)-f_{\theta}(X))q^*(\theta)(\mathit{d}\theta)\sim\mathcal{N}(0,c_f\sigma_{\varepsilon}^2),
\end{align*}
whereby the Cauchy-Schwarz inequality,
\begin{align*}
c_f=\|\int_{\Theta}(f_0(X)-f_{\theta}(X))q^*(\theta)(\mathit{d}\theta)\|^2_2\leq\mathcal{R}_1
\end{align*}
Thus, $\mathcal{R}_2=O_p(\sqrt{\mathcal{R}_1})$, and with high probability, $\mathcal{R}_2\leq C_0'\mathcal{R}_1$ for some positive constant $C_0'$ if $\lim n(r_n^*+\xi_n^*)=\infty$, or for any diverging sequence $C_0'$ if $\lim\sup n(r_n^*+\xi_n^*)\neq\infty$. Therefore,
\begin{equation}\label{eq:lem1_step_1} 
\int_{\R^T} l_n(P_0,P_{\theta})q^*(\theta)(\mathit{d}\theta)\leq C_1'(nr_n^*+\|f_{\theta^*}-f_0\|^2_{\infty})\quad\text{w.h.p.} 
\end{equation}

In the next step, we try to bound the $\KL$ divergence between $q^*$ and $\pi(\theta|\lambda)$, 
\begin{align}
&\KL\Big(q^*(\theta)\|\pi(\theta|\lambda)\Big) \nonumber\\
    \leq &\log\left(\frac{1}{\pi(\gamma^*)}\right) + \sum_{i=1}^T\KL\left[\gamma_i^*\left[\sum_{k}\phi_{k}(\theta_i^*)\mathcal{N}(\mu_k^*,\sigma_n^2)\right]+(1-\gamma_i^*)\delta_0\Big\|\gamma_i^*\mathcal{N}(0,\sigma_0^2)+(1-\gamma_i^*)\delta_0\right] \label{eq:53}\\
    = & \log\frac{1}{\lambda^{s^*}(1-\lambda)^{T-s^*}}+\sum_{i=1}^T\gamma_i^*\KL\left[\sum_k\phi_k(\theta_i^*)\mathcal{N}(\mu_k^*, \sigma_n^2) \| \mathcal{N}(0, \sigma_n^2)\right] \nonumber\\
        \leq & s^*\log\left(\frac{1}{\lambda}\right)+(T-s^*)\log\left(\frac{1}{1-\lambda}\right)+\sum_{i=1}^T \gamma_i^*\sum_{k=1}^{K}\phi_k(\theta_i^*)\KL(\mathcal{N}(\mu_k^*,\sigma_n^2)\|\mathcal{N}(0,\sigma_0^2)) \label{eq:55}\\
        \leq & s^*\log\left(\frac{1}{\lambda}\right)+(T-s^*)\log\left(\frac{1}{1-\lambda}\right)+\sum_{i=1}^T\gamma_i^*\sum_{k=1}^{K}\phi_k(\theta_i^*)\Big[\frac{1}{2}\log\frac{\sigma_0^2}{\sigma_n^2}+\frac{\sigma_n^2+(\mu_k^*)^2}{2\sigma_0^2}-\frac{1}{2}\Big] \nonumber\\
        \leq & s^*\log\left(\frac{1}{\lambda}\right)+(T-s^*)\log\left(\frac{1}{1-\lambda}\right)+\sum_{i=1}^T\gamma_i^*\Big[\frac{1}{2}\log\frac{\sigma_0^2}{\sigma_n^2}+\frac{\sigma_n^2+4}{2\sigma_0^2}-\frac{1}{2}\Big] \nonumber \\
        \leq & C_0 n r_n^* + \frac{s^*}{2}\sigma_n^2+\frac{s^*}{2\sigma_0^2}(B^2-1)+\frac{s^*}{2}\log\left(\frac{\sigma_0^2}{\sigma_n^2}\right) \label{eq:58}\\
        \leq & (C_0+1)nr_n^*+\frac{s^*}{2\sigma_0^2}B^2+\frac{s^*}{2}\log\Bigl(\frac{16n}{s^*}\log(3pN)(2BN)^{2L+2}\Bigl\{(p+1+\frac{1}{BN-1})^2+\frac{1}{(2BN)^2-1}+\frac{2}{(2BN-1)^2}\Bigr\}\Bigr) \nonumber\\
    \leq & (C_0+2)nr_n^*+\frac{s^*}{2\sigma_0^2}B^2+(L+1)s^*\log(2BN)+ \frac{s^*}{2}\log\log(3BN)+\frac{s^*}{2}\log\Bigl(\frac{n}{s^*}p^2\Bigr) \nonumber \\
    \leq &(C_0+3)nr^*_n + (L+1)s^*\log N+s^*\log\Bigl(p\sqrt{\frac{n}{s^*}}\Bigr) \nonumber\\
	\leq & C_1nr^*_n ,  \mbox{ for sufficiently large } n.\nonumber
\end{align}
where the inequality \eqref{eq:53} and \eqref{eq:55} are follows from Lemma \ref{lem:mix} and the inequality \eqref{eq:58} is because of the fact that $B > 2$ and $\sum_{i=1}^T \gamma_i^* = s^*$, thus combined with the result in equation~\eqref{eq:lem1_step_1}, we finish the proof. 
\end{proof}

\newpage

\subsection{Proof of Lemma \ref{lem:lmbound1}}   \label{sec:proof-lmbound1}
\begin{proof}

Following previous work~\cite[proof of Lemma 4.2]{bai2020efficient}, we first define the space
\begin{align*}
H_n(\theta) &= \{\theta \in \R^T:\norm{\theta}_0\leq s_n, \norm{\theta}_\infty \leq B+1\}, \\
H'_n(\theta) &= \{\theta \in \R^T:\norm{\theta}_0 > s_n, \norm{\theta}_\infty \leq B+1\}, \\
H''_n (\theta) &= \{\theta \in \R^T: \norm{\theta}_\infty > B+1\}.
\end{align*}

By the above definitions, we now have:
\begin{equation}\label{eq:lem5_def1}
    \int_{\R^T} d^2(P_\theta, P_0)\widehat{q}(\mathit{d}\theta) = \int_{H_n(\theta)} d^2(P_\theta, P_0)\widehat{q}(\mathit{d}\theta)+\int_{H'_n(\theta)} d^2(P_\theta, P_0)\widehat{q}(\mathit{d}\theta)+\int_{H''_n(\theta)} d^2(P_\theta, P_0)\widehat{q}(\mathit{d}\theta).
\end{equation}

Lemma \ref{lmdonsker} presents a variational characterization of the $\KL$ divergence, originally due to Donsker and Varadhan. The proof is available in \cite{Boucheron2013Concentration}.
\begin{lemma}{\label{lmdonsker}}
	Let $\mu$ be any probability measure and $h$ a measurable function with $e^h \in L_1(\mu)$, then
	$$
	\log \int e^{h(\eta)}\mu(d\eta) = \sup_{\rho} \left[\int h(\eta)\rho(d\eta) - \KL(\rho\|\mu)\right].
	$$
\end{lemma}

We can define the truncation of distribution $\widehat{q}(\cdot)$ on the set $H_n(\theta)$ denoted as $\check{q}(\cdot)$, (i.e. $\check{q}(\theta) = \widehat{q}(\theta) \mathbbm{1}\{\theta \in H_n(\theta)\}/\widehat{q}(H_n(\theta))$ ), similarly we can also define the $\tilde{\pi}(\theta) = \pi(\theta)\mathbbm{1}\{\theta \in H_n(\theta)\}/\widehat{q}(H_n(\theta))$.
By adopting the arguments from \cite{bai2020efficient}, and following steps analogous to those leading to Equation (17)  therein, we obtain:
\begin{equation}\label{eqn:song:1}
	\int_{H_n(\theta)} \eta(P_{\theta}, P_0)\, \widetilde{\pi}(\theta)\, \mathit{d}\theta \leq \exp\left({C_0 n \varepsilon^{*2}_n}\right), \quad \text{w.h.p.}
\end{equation}
for some constant $C_0 > 0$, where $\log\eta(P_{\theta}, P_0) = l_n(P_{\theta}, P_0) + \frac{n}{3}d^2(P_{\theta}, P_0)$.

Then, given the Lemma \ref{lem:bound_elbo} and equation \eqref{eqn:song:1}, we can show that the first term can be bounded w.h.p. as:
\begin{equation}\label{eqn:1}
	\begin{split}
	&\frac{n}{3\widehat q(H_n(\theta))}\int_{H_n(\theta)} d^2(P_{\theta}, P_0)\widehat{q}(\theta)\mathit{d}\theta \\
    &\leq Cn\varepsilon_n^{*2}+ \mbox{KL}(\check{q}(\theta)\|\widetilde{\pi}(\theta))-\int_{H_n(\theta)} l_n(P_{\theta}, P_0)\check{q}(\theta) \mathit{d}\theta\\
	&=Cn\varepsilon_n^{*2}+\frac{1}{\widehat q(H_n(\theta))}\left( \mbox{KL}(\widehat{q}(\theta)\|\pi(\theta))-\int_{\Theta} l_n(P_{\theta}, P_0)\widehat{q}(\theta) \mathit{d}\theta \right) \\
	&- \frac{1}{\widehat q(H_n(\theta))}\left( \int_{ H_n(\theta)^c} \log \frac{\widehat{q}(\theta)}{\pi(\theta)} \widehat{q}(\theta)\mathit{d}\theta -\int_{H_n(\theta)^c} l_n(P_{\theta}, P_0)\widehat{q}(\theta) \mathit{d}\theta\right)+ \log \frac{\pi(H_n(\theta))}{\widehat q(H_n(\theta))}.
	\end{split}
\end{equation}
Additionally, since $d^2(P_{\theta}, P_0) \leq 1$, the second and the third terms of Equation~\eqref{eq:lem5_def1} are bounded by:
\begin{align*}
\int_{\theta \in H_n'(\theta)} d^2(P_{\theta}, P_0)\widehat{q}(\theta)\mathit{d}\theta &\leq \int_{\theta \in H_n'(\theta)}\widehat{q}(\theta)\mathit{d}\theta =\widehat q(H_n'(\theta)), \\
\int_{\theta \in H_n''(\theta)} d^2(P_{\theta}, P_0)\widehat{q}(\theta)\mathit{d}\theta &\leq \int_{\theta \in H_n''(\theta)}\widehat{q}(\theta)\mathit{d}\theta =\widehat q(H_n''(\theta)).
\end{align*}
Substituting the bound on the first term from Equation~\eqref{eqn:1}, together with the two bounds above, into Equation~\eqref{eq:lem5_def1}, we obtain with high probability:
\begin{equation}\label{eqn:sum}
	\begin{split}
	\int d^2(P_{\theta}, P_0)\widehat{q}(\theta)\mathit{d}\theta 
	\leq \;&\;  3 \widehat q(\Theta_n)C\varepsilon^{*2}_n + \frac{3}{n}\left( {\KL}(\widehat{q}(\theta)\|\pi(\theta))-\int_{H_n(\theta)} l_n(P_{\theta}, P_0)\widehat{q}(\theta) \mathit{d}\theta \right) \\
	&+ \frac{3}{n} \int_{H_n'(\theta)^c} l_n(P_{\theta}, P_0)\widehat{q}(\theta) \mathit{d}\theta+ \frac{3}{n} \int_{ H_n'(\theta)^c} \log \frac{\pi(\theta)}{\widehat{q}(\theta)} \widehat{q}(\theta)\mathit{d}\theta \\
    &+ \frac{3\widehat q(H_n(\theta))}{n}\log \frac{\pi(H_n(\theta))}{\widehat q(H_n(\theta))}+\widehat q(H_n'(\theta))+\widehat q(H_n''(\theta)).
	\end{split}
\end{equation}
Following the procedure in~\cite[Lemma 4.2, equation 20]{bai2020efficient}, we can show with high probability that:
\begin{align*}
\int  d^2(P_{\theta}, P_0)\widehat{q}(\theta)\mathit{d}\theta  
    &\leq 3\widehat q(H_n)C\varepsilon^{*2}_n + \frac{3}{n}\left( \KL(\widehat{q}(\theta)\|\pi(\theta))-\int l_n(P_{\theta}, P_0)\widehat{q}(\theta) \mathit{d}\theta \right) + \frac{3}{n} \int_{H_n^c} l_n(P_{\theta}, P_0)\widehat{q}(\theta) \mathit{d}\theta \\
    &\qquad + \frac{3}{n} \int_{ H_n^c} \log \frac{\pi(\theta)}{\widehat{q}(\theta)} \widehat{q}(\theta)\mathit{d}\theta + \frac{3\widehat q(H_n)}{n}\log \frac{\pi(H_n)}{\widehat q(H_n)}+\widehat q(H_n'(\theta))+\widehat q(H_n''(\theta)) \\
    &= 3\widehat q(H_n)C\varepsilon^{*2}_n + \frac{3}{n}\left( \KL(\widehat{q}(\theta)\|\pi(\theta))-\int_{\R^T} l_n(P_{\theta}, P_0)\widehat{q}(\theta) \mathit{d}\theta \right)+ \frac{3}{n} \int_{H_n^c} l_n(P_{\theta}, P_0)\widehat{q}(\theta) \mathit{d}\theta \\
	&\qquad + \frac{3}{n} \int_{ H_n^c} \log \frac{\pi(\theta)}{\widehat{q}(\theta)} \widehat{q}(\theta)\mathit{d}\theta + \frac{3\widehat q(H_n)}{n}\log \frac{\pi(H_n)}{\widehat q(H_n)}+\widehat q(H'_n)+\widehat{q}(H''_n) \\
    &\leq C\varepsilon^{*2}_n + \frac{3}{n}\left( \KL(\widehat{q}(\theta)\|\pi(\theta))-\int_{\R^T} l_n(P_{\theta}, P_0)\widehat{q}(\theta) \mathit{d}\theta \right) + O(1/n) + \widehat{q}(H''_n),
\end{align*}
where $C$ is some constant. Next, we show that $\widehat{q}(H''_n(\theta)) = O(\varepsilon_n^{*2})$.

\begin{lemma} \label{lem:bound}Given the Condition \ref{cond:bound}, $\widehat{q}(H''_n(\theta)) = O(\varepsilon_n^{*2})$ holds.
\end{lemma}

\begin{proof}
Let $\widehat{\sigma}_m = \max\{\widehat{\sigma}_1, \dots, \widehat{\sigma}_K\}$, then by definition of the variational distribution $q(\theta)$,we can know that:
\begin{equation*}
    \widehat{q}(H''_n(\theta)) \leq \sum_{i=1}^T \widehat{q}(|\theta_i| > (B+1)) = \sum_{i=1}^T \widehat{q}(\theta_i> (B+1)) + \widehat{q}(\theta_i<-(B+1)).
\end{equation*}
By the definition of variational distribution $\widehat{q}(\cdot)$, we know:
\begin{align*}
\widehat{q}(\theta_i > B+1) &\leq \int_{B+1}^\infty \frac{1}{\sqrt{2\pi\sigma_m^2}}\exp\Big(\frac{-(t-B)^2}{2\sigma_m^2}\Big)\mathit{d}t, \\
\widehat{q}(\theta_i < -(B+1)) &\leq \int^{-(B+1)}_{-\infty} \frac{1}{\sqrt{2\pi\sigma_m^2}}\exp\Big(\frac{-(t-B)^2}{2\sigma_m^2}\Big)\mathit{d}t. 
\end{align*}
And by Chernoff bound and the fact that $\sigma_m^2 \leq \frac{1}{2\log(T/\varepsilon_n^{*2})}$, we can have:
\begin{equation*}
    \int_{B+1}^\infty \frac{1}{\sqrt{2\pi\sigma_m^2}}\exp\Big(\frac{-(t-B)^2}{2\sigma_m^2}\Big)\mathit{d}t \leq \frac{1}{2}\exp\Big(-\frac{1}{2\sigma_m^2}\Big) \leq \frac{\varepsilon_n^{*2}}{2T}. 
\end{equation*}
Thus we show that $\widehat{q}(H''_n) \leq \varepsilon_n^{*2}$.
\end{proof}

Then by Lemma \ref{lem:bound}, we complete the proof. 
\end{proof}

\newpage

\section{Implementation of \method} \label{apx:implement}

\noindent\textbf{Pretrained model setting.}
Taking the compression of the Llama3.2-1B model as an example, we first download the pre-trained model from Hugging Face\footnote{\url{https://huggingface.co/meta-llama/Llama-3.2-1B}} using the Python package ``\texttt{transformers}''. We then fine-tune the model on the considered SST-2 task before applying our \method for compression. We find that omitting the fine-tuning step significantly degrades the performance of \method.

For ResNet models, we use publicly available pre-trained models obtained by the Python package ``\texttt{timm}'' on the CIFAR-10 and CIFAR-100 datasets and directly apply our \method for compression. 
For the BERT-base model, we download the pre-trained model from Hugging Face\footnote{\url{https://huggingface.co/huggingface-course/bert-finetuned-squad}} using the Python package ``\texttt{transformers}''.

\noindent\textbf{Initialization.}
To initialize the learnable parameters of our \method method, denoted as $\{\mu_k, \sigma_k, \pi_k\}_{k=1}^K$, we employ the K-means algorithm. Specifically, the DNN weights of a given layer are first clustered into $K$ groups. For each group $G_k$, the mean $\mu_k$ and standard deviation $\sigma_k$ are computed as the empirical statistics of the weights in that group, while the mixture coefficient $\pi_k$ is set to the proportion of weights in group $k$ relative to the total number of weights in the layer.

We assume that K-means yields $K$ disjoint groups of weights, denoted as $\{G_1, \dots, G_K\}$, such that $\bigcup_{k=1}^K G_k$ covers all weights in the selected layer. The initial parameters are then defined as:
\begin{align*}
\mu_k = \frac{1}{|G_k|} \sum_{\theta_i \in G_k} \theta_i, \qquad
\sigma_k = \sqrt{\frac{1}{|G_k| - 1} \sum_{\theta_i \in G_k} (\theta_i - \mu_k)^2}, \qquad
\pi_k = \frac{|G_k|}{\sum_{j=1}^K |G_j|}.
\end{align*}

\noindent\textbf{Implementing marginal $q(\theta_i)$.}
To make the distribution differentiable, we reparameterize $\tilde{\lambda}_i$ (as defined in Equation~\ref{eq:marginal_vb}) using the following equation:
\begin{align} \label{eq:lambda-prime}
\tilde{\lambda}_i = \frac{\exp(\tilde{s}_i / \tau')}{1 + \exp(\tilde{s}_i / \tau')},
\end{align}
where the temperature $\tau'$ is set to a fixed constant to stabilize the training process. To better exploit the learned pruning parameters in later stages of training, we halve $\tau'$ after completing half of the total training steps to stabilize the training and sharpen the retention probabilities around their learned optima.

\noindent\textbf{Hyperparameter Configuration.}
The number of Gaussian components $K$ is not fixed across experiments; it is chosen per model and can be read off the \texttt{Bits} column of each experiment table, since $\texttt{Bits}=\log_2 K$. Specifically, we use $K=4$ ($2$ bits) for the ResNet models on CIFAR-10 in Table~\ref{tab:cifar10}, $K=16$ ($4$ bits) for BERT-base on SQuAD v1.1 in Table~\ref{tab:bert_squad}, and $K=64$ ($6$ bits) for Llama3.2-1B and Qwen2.5-0.5B on SST-2 in Table~\ref{tab:sst}. The ablation studies state their own $K$ in the corresponding table or figure caption. A larger $K$ reduces the performance drop at the cost of a lower compression rate; this trade-off is analyzed in the case study shown in Figure~\ref{fig:quantization_compare}. For the experiments not reporting the Bayesian Averaging number $M$, the default is set to 4. All models are trained on an NVIDIA H100 GPU with 80 GB of memory.

During training and testing, for ResNet-18, ResNet-20, ResNet-32, ResNet-50, ResNet-56, BERT-base, Llama3.2-1B, and Qwen2.5-0.5B models, the settings are:
\begin{itemize}[leftmargin=*,itemsep=2pt]
    \item \textbf{Training Time:} Approximately 30 minutes for ResNet models (ResNet-18 through ResNet-56); 4 hours for BERT-base; 24 hours for both Llama3.2-1B and Qwen2.5-0.5B.
    \item \textbf{Optimizer:} AdamW is used consistently across all models.
    \item \textbf{Quantization Learning Rate:} $5\times10^{-4}$ for ResNet-18; $5\times10^{-5}$ for all other models.
    \item \textbf{Pruning Learning Rate:} Fixed at $0.012$ for all models.
    \item \textbf{Temperature Hyperparameters:}   $\tau$ in Equation~\eqref{eq:compute-phi-k} is set to $5\times10^{-4}$ and $\tau'$ in Equation~\eqref{eq:lambda-prime} is set to $0.0125$.
    \item \textbf{Pruning Schedule:} A polynomial schedule is used for all models.
\end{itemize}

\newpage

\section{Experiment Settings}\label{apx:exp-set}

\subsection{Experiment settings for benchmark with all baselines}

\noindent\textbf{Benchmark compression on ResNet models.}  We present experiments using ResNet architectures on the CIFAR-10 and CIFAR-100 datasets. 
When compressing ResNet models, our method requires fine-tuning over the training dataset, completing the compression process within $10$ epochs. 
To achieve high compression rates, we represent each layer's weights with $4$ components (i.e., $K = 4$ for each layer) and apply a sparsity level of $50\%$. As shown in Table 1, our method achieves 32× compression on all three ResNet models, with reported accuracy drops below 1.5 percentage points.
For example, compressing ResNet-32 by a factor of $32\times$ yields a minimal accuracy reduction of $1.29\%$. 
Additionally, we compress ResNet-56 by a factor of $32\times$, observing an accuracy drop of only $0.84\%$.
Compared to other methods, our approach achieves much higher compression rates with smaller decreases in accuracy. 

\noindent\textbf{Benchmark compression on BERT-base model.} We further investigate our compression method on attention-based models. We apply our compression model on the BERT-base~\citep{devlin2018bert}  model and test it on the SQuAD V1.1 dataset \citep{rajpurkar2016squad}. Similarly, we consider the F1 score drop and compression rate as the evaluation metrics. During the compression process, the BERT model is fine-tuned on the training dataset, with the entire procedure completed within $3$ epochs.

We compressed the BERT model using $K=16$ Gaussian components and pruned $75\%$ of its parameters, leading to a $32\times$ compression rate. We employed layer-wise quantization combined with unstructured pruning to attain these results.

\noindent\textbf{Benchmark compression on Llama and Qwen models.}
Due to hardware limitations, we cannot run very large-scale LLMs, which are Llama3.1-8B and Qwen2.5-7B.

\subsection{Experiment settings for ablation studies for \method method}\label{apx:gaussian_prior}

\noindent\textbf{Impact of different priors.}  For comparison, we consider a zero-mean Gaussian distribution as the prior and replace the delta distribution with a Gaussian distribution in the variational family.
That is, any $q'(\theta) \in \mathcal{F}'$ has the form:
\begin{align*}
    \tilde{\theta}_i | \gamma_i \sim \gamma_i  \sum_{k=1}^{K} \phi_k(\theta_i) \mathcal{N}(\mu_k, \sigma^2_k) +(1-\gamma_i) \mathcal{N}(0, \sigma^2_{0}), &\qquad  \gamma_i \sim \texttt{Bern}(\tilde{\lambda}_i).
\end{align*}
Based on this, we can get the modified marginal variational distribution $q'(\tilde{\theta}_i)$ as:
\begin{equation}
    q'(\tilde{\theta}_i)=\tilde{\lambda}_i  \sum_{k=1}^{K} \phi_k(\theta_i) \mathcal{N}(\mu_k, \sigma^2_k) + (1-\tilde{\lambda}_i) \mathcal{N}(0, \sigma^2_{0}).
\end{equation}
Thus, following the same reasoning and derivation, we get the equation \eqref{eq:approx_elbo}, The Gaussian prior can be obtained by:
\begin{align}
    \nonumber \Omega' &= -\E_{q'(\tilde{\theta})}\big[\log p(D|\tilde{\theta})\big]+\sum_{i=1}^{T}\KL\left(q'(\tilde{\theta}_i) \| \mathcal{N}(0, \sigma^2_0)\right) \\
    \nonumber &= -\E_{q'(\tilde{\theta})}\big[\log p(D|\tilde{\theta})\big]+\sum_{i=1}^{T}\KL\left(q'(\tilde{\theta}_i) \| (\tilde{\lambda}_i +(1-\tilde{\lambda}_i))\mathcal{N}(0, \sigma^2_0)\right) \\
    &\leq -\E_{q'(\tilde{\theta})}\big[\log p(D|\tilde{\theta})\big] + \sum_{i=1}^{T}\tilde{\lambda}_i\KL\left(\sum_{k=1}^{K} \phi_k(\theta_i) \mathcal{N}(\mu_k, \sigma^2_k) \Big\| \mathcal{N}(0, \sigma_0^2)\right). \tag{Gaussian prior}
    \label{eq:elbo_upperbound_normal}
\end{align}
We compare the impact of the above Gaussian prior with the proposed Spike-and-GMM priors and summarize the result in Table~\ref{tab:prior_compare}.

\noindent\textbf{Description of Baselines.} For the following lines of baselines, we use the reported results in their papers:
\begin{itemize}
\item Optimal Brain Surgeon (OBS)~\citep{hassibi1993optimal} selects weights for removal from a trained neural network using second-order information. 
\item BitSplit~\citep{wang2020towards} incrementally constructs quantized values using a squared error metric based on residual errors. 
\item AdaQuant~\citep{hubara2021accurate} utilizes STE for direct optimization. 
\item BRECQ~\citep{li2021brecq} integrates Fisher information into the optimization process and focuses on the joint optimization of layers within individual residual blocks.

    \item Exact Optimal Brain Quantization (OBQ)~\citep{frantar2022optimal} adapts second-order weight pruning methods to quantization tasks.
    \item GPTQ~\citep{frantar2022gptq} employs second-order information for error compensation on calibration sets to speed up generative models.
\end{itemize}

We adopt their implemented code and use the same setting for training and testing:
\begin{itemize}
    \item AWQ~\citep{lin2024awqactivationawareweightquantization} implements activation-aware quantization, selectively bypassing the quantization of key weights\footnote{\url{https://github.com/mit-han-lab/llm-awq}}. This method is training-free and does not need extra training on the selected dataset.
\item  DGMS~\citep{dong2022finding} is an automated quantization method that utilizes  Mixtures of Gaussians to avoid the aforementioned problem\footnote{\url{https://github.com/RunpeiDong/DGMS}}.  We use their codebase and configure it with the same hyperparameters. Their algorithm is trained on the same dataset for fairness of comparison.
\end{itemize}

\noindent\textbf{Definition of Evaluation Metrics.}
Let $K$ denote the number of shared weight vectors. Then, the metric \texttt{Bits} is defined as $\log_2 K$.

\subsection{Long-tailed Full-precision Weight Distribution of Llama3.2 and Qwen2.5 Models}\label{apx:extended_result}

We present the visualization of the long-tailed weight distributions for the Llama3.2 model in Figures~\ref{fig:llama3-weight-p1} and~\ref{fig:llama3-weight-pa2}, and for the Qwen2.5 model in Figures~\ref{fig:qwen-weight-p1} and~\ref{fig:qwen-weight-p2}.

\begin{figure}[!ht]
    \centering
\includegraphics[width=0.75\linewidth]{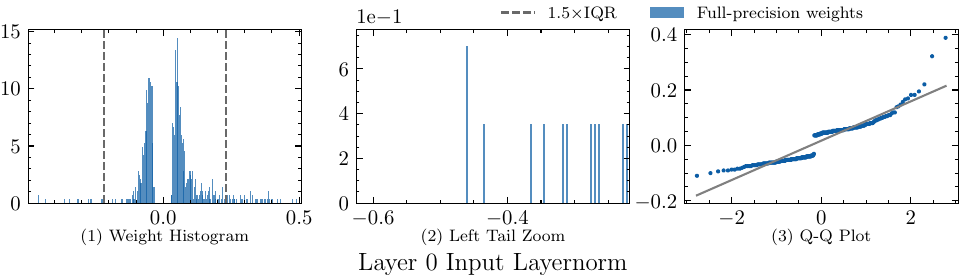}
\includegraphics[width=0.75\linewidth]{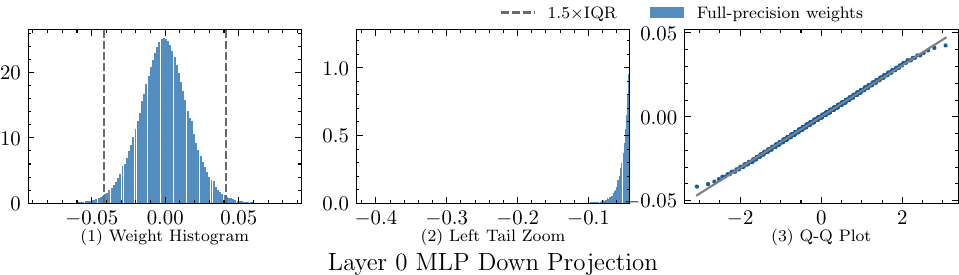}
\includegraphics[width=0.75\linewidth]{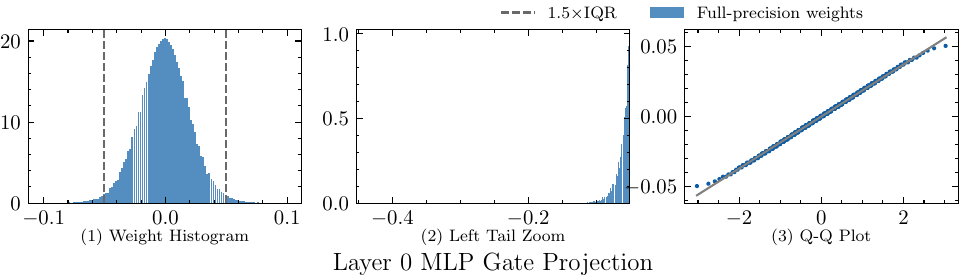}
\includegraphics[width=0.75\linewidth]{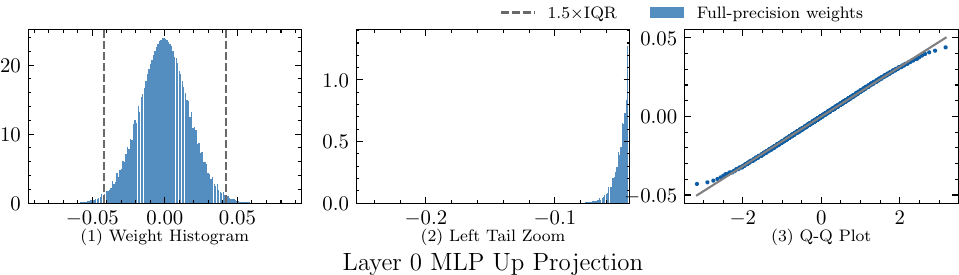}
\includegraphics[width=0.75\linewidth]{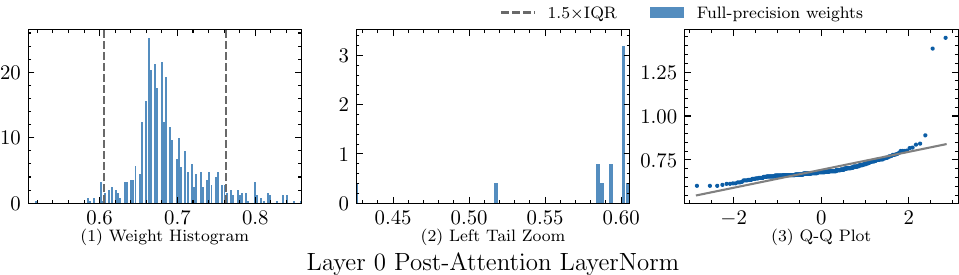}
\includegraphics[width=0.75\linewidth]{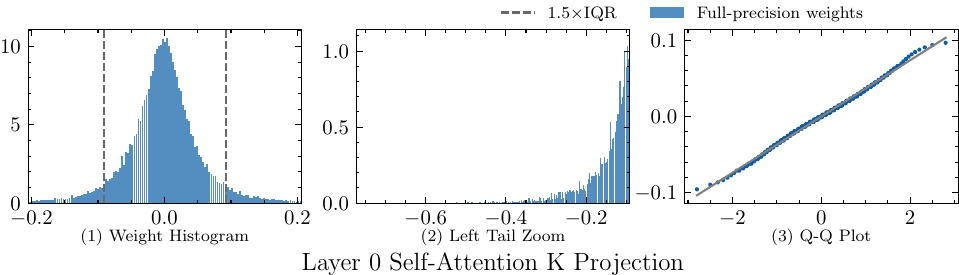}
\vspace{-8pt}
    \caption{Long-tailed weight distributions of different layers in Qwen2.5 model (part 1).}
    \label{fig:qwen-weight-p1}
\end{figure}

\begin{figure}[!ht]
    \centering
\includegraphics[width=0.75\linewidth]{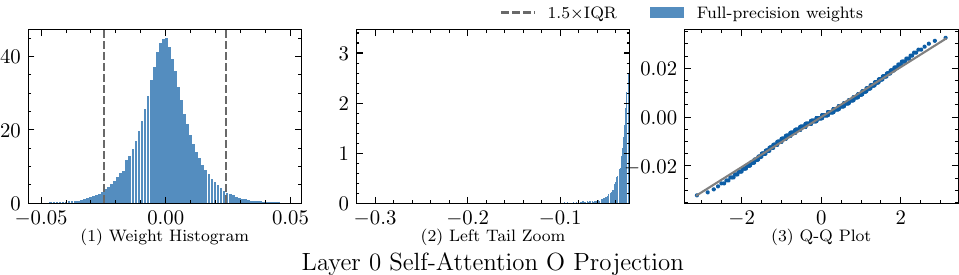}
\includegraphics[width=0.75\linewidth]{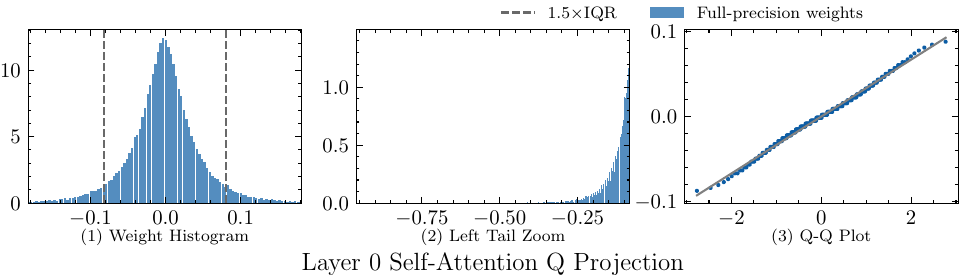}
\includegraphics[width=0.75\linewidth]{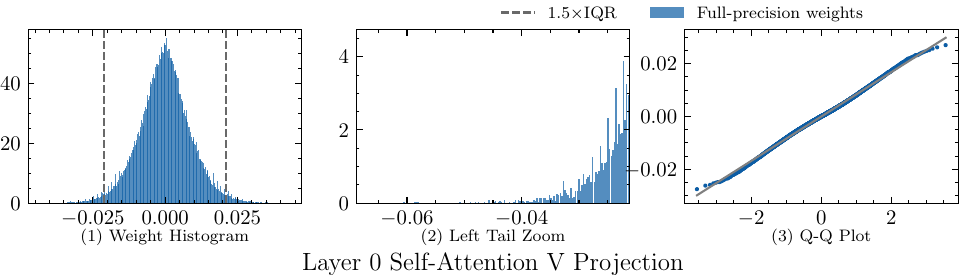}
\includegraphics[width=0.75\linewidth]{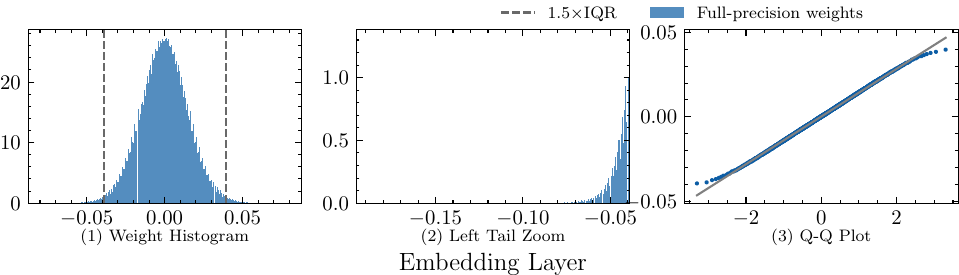}
\includegraphics[width=0.75\linewidth]{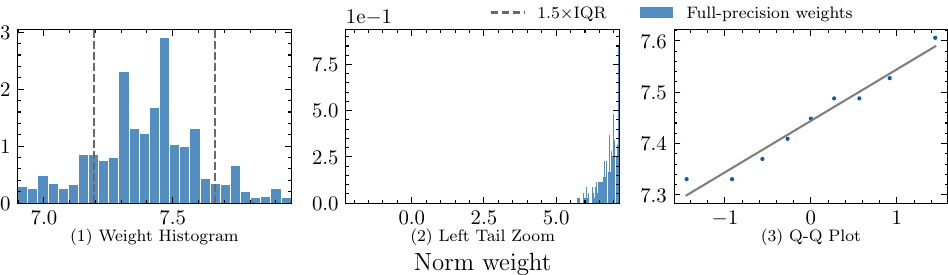}
\includegraphics[width=0.75\linewidth]{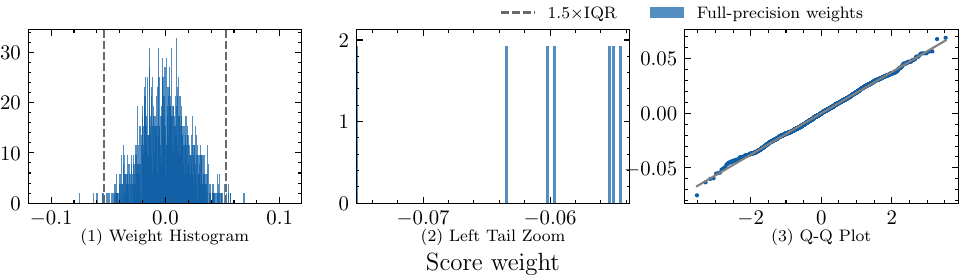}
\vspace{-8pt}
    \caption{Long-tailed weight distributions of different layers in Qwen2.5 model (part 2).}
    \label{fig:qwen-weight-p2}
\end{figure}

\begin{figure}[!ht]
    \centering
\includegraphics[width=0.75\linewidth]{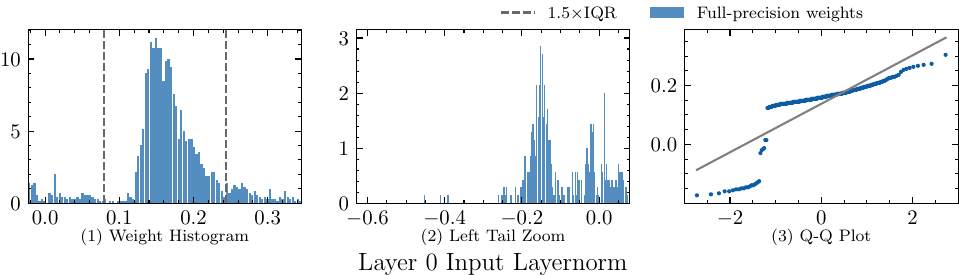}
\includegraphics[width=0.75\linewidth]{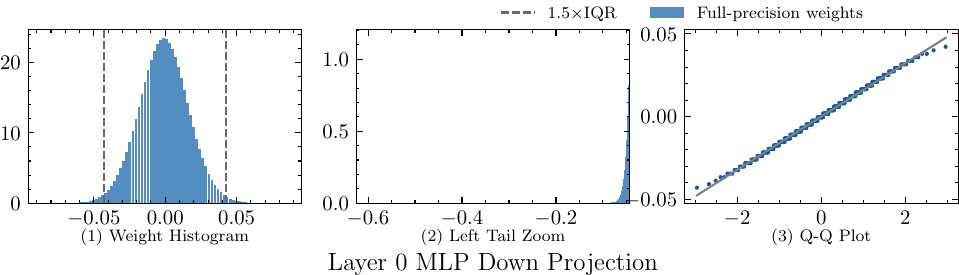}
\includegraphics[width=0.75\linewidth]{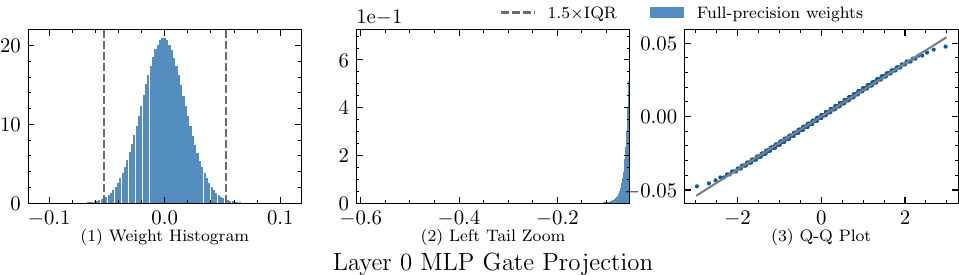}
\includegraphics[width=0.75\linewidth]{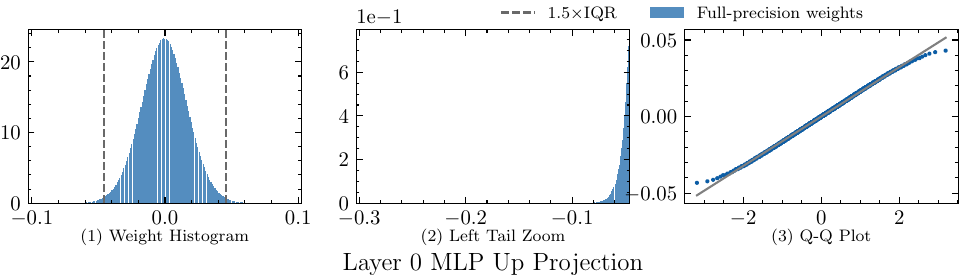}
\includegraphics[width=0.75\linewidth]{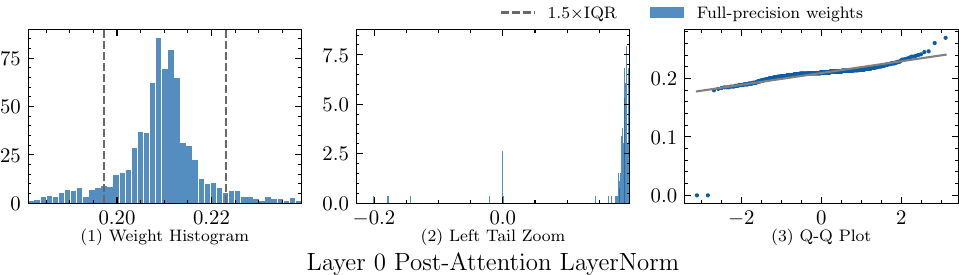}
\includegraphics[width=0.75\linewidth]{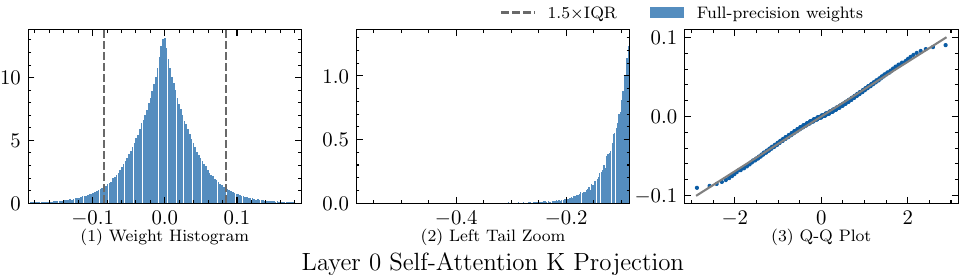}
\vspace{-8pt}
    \caption{Long-tailed weight distributions of different layers in Llama3.2 model (part 1).}
    \label{fig:llama3-weight-p1}
\end{figure}

\begin{figure}[!ht]
    \centering
\includegraphics[width=0.75\linewidth]{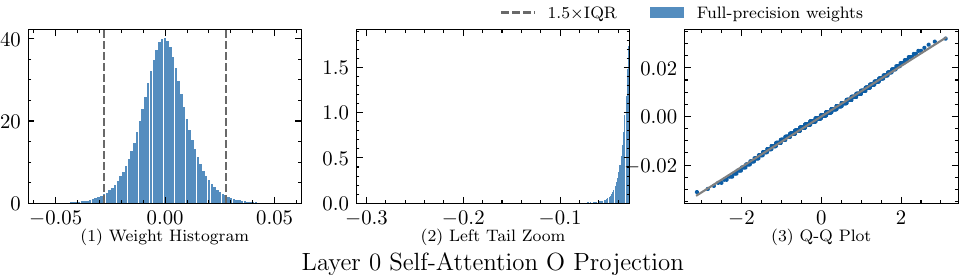}
\includegraphics[width=0.75\linewidth]{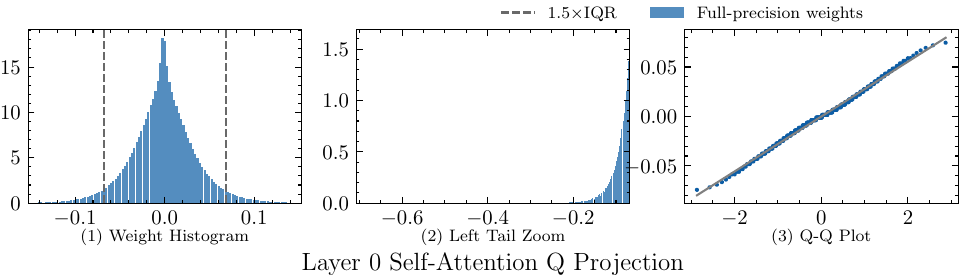}
\includegraphics[width=0.75\linewidth]{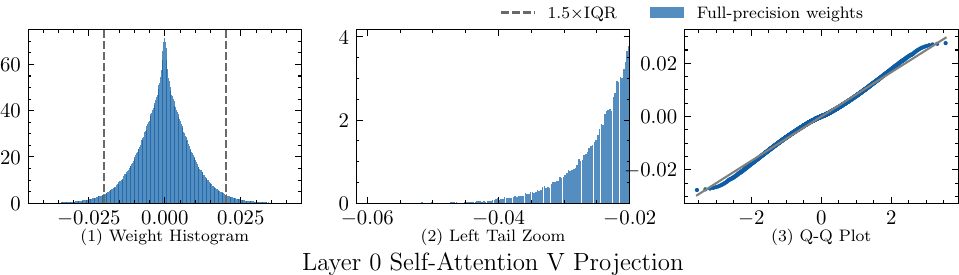}
\includegraphics[width=0.75\linewidth]{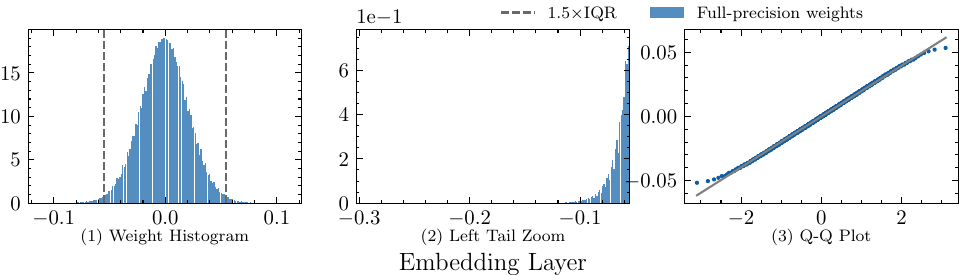}
\includegraphics[width=0.75\linewidth]{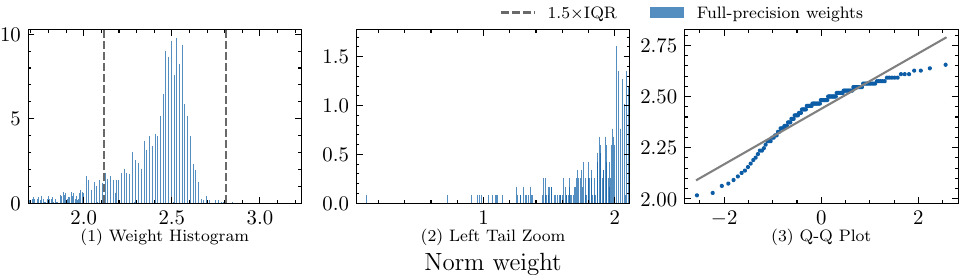}
\includegraphics[width=0.75\linewidth]{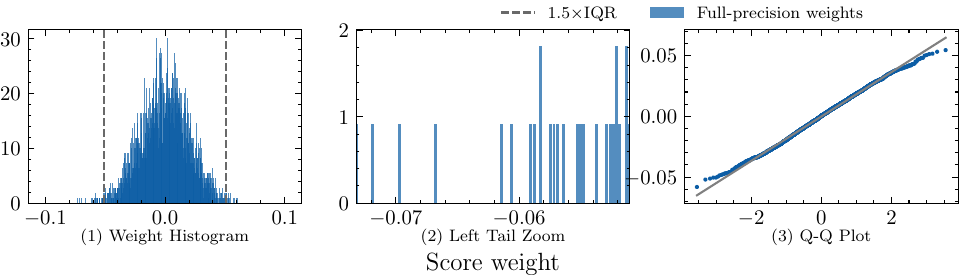}
\vspace{-8pt}
    \caption{Long-tailed weight distributions of different layers in Llama3.2 model (part 2).}
    \label{fig:llama3-weight-pa2}
\end{figure}

% \begin{figure}[!ht]
%     \centering
%     \includegraphics[width=.9\linewidth]{figs/weight_distribution/compare/(a) Self-attention layer K projection weight distribution.pdf}
%     \includegraphics[width=.9\linewidth]{figs/weight_distribution/compare/(b) Self-attention layer O projection weight distribution.pdf}
%     \includegraphics[width=.9\linewidth]{figs/weight_distribution/compare/(c) Self-attention layer Q projection weight distribution.pdf}
%     \includegraphics[width=.9\linewidth]{figs/weight_distribution/compare/(d) Self-attention layer V projection weight distribution.pdf}
%     \caption{For the compressed weight distributions of the $K$, $O$, $Q$, and $V$ matrices in the self-attention layer of the Llama3.2-1B model, \method using the outlier-aware window strategy (\textbf{Left}) more effectively preserves the characteristics of the full-precision weight distribution compared to the equal-sized window strategy (\textbf{Middle}). This improvement is particularly noticeable in the left tail region, as highlighted in (\textbf{Right}).}
%     \label{fig:attention-compare-extend}
% \end{figure}

% \noindent\textbf{Long tail properties of weight distribution.} 

\end{document}